\documentclass{article}
\usepackage[margin=2cm]{geometry}

\usepackage[utf8]{inputenc} 
\usepackage[T1]{fontenc}    
\usepackage{hyperref}       
\usepackage{url}            
\usepackage{booktabs}       
\usepackage{amsfonts}       
\usepackage{nicefrac}       
\usepackage{microtype}      
\usepackage{xcolor}         
\definecolor{brickred}{RGB}{203,65,84}
\usepackage{enumitem}       
\usepackage{algorithmic}    
\usepackage{algorithm}      
\usepackage{graphicx}
\usepackage{subcaption}
\usepackage{adjustbox}      
\usepackage{wrapfig}        
\usepackage{natbib}
\usepackage{array}

\newcolumntype{?}{!{\vrule width 1pt}}

\usepackage{amsmath}
\usepackage{amssymb}
\usepackage{mathtools}
\usepackage{amsthm}
\usepackage{bm}

\usepackage[capitalize,noabbrev]{cleveref}

\theoremstyle{plain}
\newtheorem{theorem}{Theorem}[section]

\theoremstyle{definition}

\theoremstyle{remark}

\newcommand{\newterm}[1]{{\bf #1}}

\def\eqref#1{equation~\ref{#1}}

\def\1{\bm{1}}

\def\eps{{\epsilon}}

\def\va{{\bm{a}}}

\def\vc{{\bm{c}}}

\def\vk{{\bm{k}}}
\def\vl{{\bm{l}}}

\def\vo{{\bm{o}}}

\def\vq{{\bm{q}}}

\def\vs{{\bm{s}}}

\def\vv{{\bm{v}}}

\def\vx{{\bm{x}}}
\def\vy{{\bm{y}}}
\def\vz{{\bm{z}}}

\DeclareMathAlphabet{\mathsfit}{\encodingdefault}{\sfdefault}{m}{sl}
\SetMathAlphabet{\mathsfit}{bold}{\encodingdefault}{\sfdefault}{bx}{n}

\newcommand{\R}{\mathbb{R}}

\newcommand{\RMS}{\mathrm{RMS}}
\newcommand{\MLP}{\mathrm{MLP}}
\newcommand{\Attn}{\mathrm{Attn}}
\newcommand{\lnorm}[2]{\Vert #1 \Vert_{#2}}
\newcommand{\ltwonorm}[1]{\lnorm{#1}{2}}
\newcommand{\innerp}[2]{\langle{#1, #2}\rangle}

\renewcommand{\tilde}{ }

\newcommand{\methodname}{Recurrent Transformer}
\newcommand{\methodnames}{Recurrent Transformers}
\newcommand{\methodnameabv}{\textsc{RT}}

\newcommand{\qkRMS}{\mathrm{\textcolor{magenta}{RMS}}}
\newcommand{\temp}{\textcolor{red}{\mathrm{temp}}}
\newcommand{\zmark}[1]{\textcolor{blue!55!black}{#1}}

\usepackage{amsthm}

\theoremstyle{plain}
\newtheorem{thm_samy}{Theorem}

\title{
The Recurrent Transformer:\\
Greater Effective Depth and Efficient Decoding
}

\author{
\centering
\begin{tabular}{c}
Costin-Andrei Oncescu\thanks{Correspondence to: concescu@g.harvard.edu} \quad
Depen Morwani \quad
Samy Jelassi \\
[0.25em]
Alexandru Meterez \quad
Mujin Kwun \quad
Sham Kakade \\
[0.5em]
Harvard University
\end{tabular}
}
\date{}

\begin{document}
\maketitle
\begin{abstract}
Transformers process tokens in parallel but are temporally shallow: at position $t$, each layer attends to key--value pairs computed based on the previous layer, yielding a depth capped by the number of layers. Recurrent models offer unbounded temporal depth but suffer from optimization instability and historically underutilize modern accelerators. We introduce the \emph{Recurrent Transformer}, a simple architectural change where \emph{each layer} attends to key--value pairs computed off its own activations, yielding layerwise recurrent memory while preserving standard autoregressive decoding cost. We show that the architecture can emulate both (i) a conventional Transformer and (ii) token-to-token recurrent updates under mild assumptions, while avoiding optimization instability.
Naively, prefill/training appears bandwidth-bound with effective arithmetic intensity near $1$ because keys and values are revealed sequentially; we give an exact tiling-based algorithm that preserves the mathematical computation while reducing HBM traffic from $\Theta(N^2)$ to $\Theta(N\log N)$, increasing effective arithmetic intensity to $\Theta(N/\log N)$ for sequence length $N$.
On 150M and 300M parameter C4 pretraining, \methodname s improve cross-entropy over a parameter-matched Transformer baseline and achieve the improvement with fewer layers (fixed parameters), suggesting that recurrence can trade depth for width, thus reducing KV cache memory footprint and inference latency. Code is available at \url{https://github.com/geniucos/recurrent-transformer}
\end{abstract}

\section{Introduction}

Transformers \citep{vaswani2017attention} are highly effective sequence models, but their computation across positions is structurally shallow: within each layer, position $t$ attends to key--value pairs computed from the previous layer embeddings, allowing essentially at most one interaction per layer between any pair of positions. A growing body of theory studies the fundamental limitations implied by bounded depth in attention models, including circuit-complexity characterizations of what low-depth Transformers can and cannot represent \citep{merrill2022saturated,liu2022shortcuts}. These perspectives motivate architectures that achieve greater effective depth.

We introduce the \newterm{\methodname} (\methodnameabv), a simple modification of how key--value pairs are computed that makes each layer temporally recurrent. In a standard Transformer, at layer $\ell$, the key--value pair at position $t$ is computed from the layer-$(\ell-1)$ representation at that position and can then be attended to by later positions $t' > t$. In the \methodname, by contrast, the key--value pair at position $t$ in layer $\ell$ is computed from that position's output at layer $\ell$, rather than from its layer-$(\ell-1)$ representation. Consequently, a later position $t < t'$ at layer $\ell$ attends to a representation at $t$ that already reflects layer $\ell$ attention and MLP computation. Importantly, \methodname{} performs this recurrence separately within each layer, so each layer maintains its own key--value memory. This differs from the Feedback Transformer~\citep{fan2020feedback}, which uses a shared memory across layers, and this layerwise separation is a key reason why our architecture can be implemented efficiently.

We motivate \methodname{}'s design through lenses of representation, optimization and computational efficiency:

\begin{figure*}[t]
  \centering
  \includegraphics[width=\textwidth]{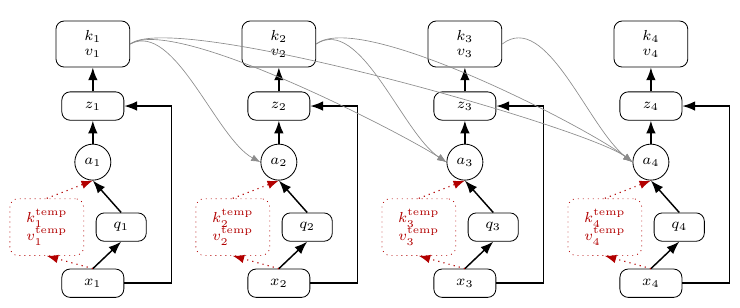}
  \caption{One layer of the \methodname{} mapping input embeddings $\vx_1\ldots\vx_4$ to output embeddings $\vz_1\ldots \vz_N$. Notice how the \emph{persistent} key--value pairs are a function of the layer's output and are used for all subsequent attention computations. The \emph{temporary} key--value pairs are only used at the time they are computed and then discarded. They are only used to avoid ill-defined attention since, for example, $\va_2$ cannot attend to $(\vk_2, \vv_2)$ as that indirectly depends on it. This is in contrast to a vanilla Transformer that uses these same key--value pairs for all subsequent attention computation as well.}
  \label{fig:rt-arch}
\end{figure*}

\paragraph{(i) Representational perspective.}
\methodnames{} retains per-token key--value memory just like a Transformer, but increase the space of computations that can be expressed within a single layer by allowing later positions to attend to representations that have already undergone attention and MLP processing. Under mild assumptions, \methodnames{} can emulate standard Transformer behavior; conversely, by restricting attention to the previous position, they can implement token-to-token recurrent computation. This positions \methodname{} between fully parallel attention and purely recurrent state-space computation, while avoiding a capped-memory bottleneck.

\paragraph{(ii) Training Stability.}
Viewing the model as a directed computation graph over positions, classical RNNs transmit information from position $i$ to $j$ only through the length-$(j-i)$ chain of intermediate states. The potentially large length of such paths gives rise to vanishing and exploding gradient phenomena \citep{bengio1994long,pascanu2013difficulty}, making it hard to ensure information flow between distant positions. \methodname{} alleviates this by introducing many additional multi-hop paths, corresponding to repeated attend+MLP applications across positions within a layer, while still permitting direct one-hop attention interactions between any two positions. In practice, we find that this architecture, together with appropriate normalization before key--value computation and standard depth-wise residual scaling \citep{bordelon2023depthwise,yang2023tensor}, trains stably. We expand on this view, and on why exploding gradients are not expected to be an issue, in Section~\ref{sec:paths}.

\paragraph{(iii) Training-time efficiency.}
A naive implementation of \methodname{} training/prefill is sequential in position and appears bandwidth-bound: keys and values are revealed one position at a time, and each query must aggregate over a linearly-growing prefix, leading to a very low effective arithmetic intensity -- $\Theta(1)$ -- under the Roofline model \citep{williams2009roofline}. We give an \emph{exact} tiling algorithm that preserves the mathematical attention computation while reorganizing memory movement, reducing high-bandwidth memory (HBM) traffic from $\Theta(N^2)$ to $\Theta(N\log N)$ and raising effective arithmetic intensity to $\Theta(N/\log N)$. Our key observation is that, during training/prefill, the full sequence of queries is available in advance even though persistent key--value pairs are revealed causally. This makes it possible to reorganize the computation into a tiled schedule, in the spirit of Flash Inference \citep{flashInference}, that reuses each revealed key--value tile across many future queries before it is evicted from fast memory. The final algorithm interleaves attention blocks and MLP computation while employing the same methodology as~\cite{rabe2021blockAttention,dao2022flashattention} to accumulate attention contribution. 



\paragraph{(iv) Depth to inference efficiency.}
Crucially, the additional effective temporal depth can translate into a better depth--width tradeoff: at fixed parameter count, achieving the same quality with fewer layers reduces the amount of stored key--value state and the corresponding decode-time memory traffic. Our experiments support this regime, with shallower \methodname{} models outperforming deeper Transformer baselines.  

\paragraph{Contributions.} 
\begin{itemize}
    \item In Section~\ref{sec:arch}, we propose the \methodname{} (\methodnameabv), a layerwise recurrent attention architecture that computes each layer’s key--value pairs from that layer’s outputs rather than from the previous layer's representations. 
    \item In Section~\ref{sec:repr}, we provide representational arguments showing \methodname{} can emulate standard self-attention behavior and can implement token-to-token recurrent computation via attention concentration under mild assumptions.
    \item In Section~\ref{sec:paths}, we provide a path-based analysis of training stability in \methodname{}, showing how the architecture combines additional multi-hop computation with direct one-hop attention paths, and giving theoretical evidence in a simplified setting that neither exploding gradients nor vanishing gradients are expected under appropriate scaling.
    \item In Section~\ref{sec:tiling}, we provide an \emph{exact}, IO-aware tiling algorithm for prefill/training that preserves the mathematical attention computation while reducing memory traffic from $\Theta(N^2)$ to $\Theta(N\log N)$ and increasing effective arithmetic intensity from $\Theta(1)$ to $\Theta(N/\log N)$.
    \item In Section \ref{sec:comp_challenges}, we outline various computational challenges and design choices required to make \methodname{} training more efficient and practical.
    \item In Section~\ref{sec:expts}, we present empirical results on 300M-parameter C4 pretraining showing improved cross-entropy over parameter-matched Transformer baselines and favorable depth--width tradeoffs at fixed parameter count (as shown in Figure \ref{fig:c4-300m}). In particular, \methodname{} with $6$ layers performs comparably to $12$ layers (fixed parameters), reducing KV cache size by approximately 30\% and lowering decode-time memory traffic, thereby improving inference efficiency. Additional results for the 150M-parameter model are provided in Appendix~\ref{app:150m-pretrain}.
\end{itemize}

\begin{figure}[t]
\centering

\begin{minipage}[t]{0.5\linewidth}
\vspace{0pt}
  \centering
  \includegraphics[width=\linewidth]{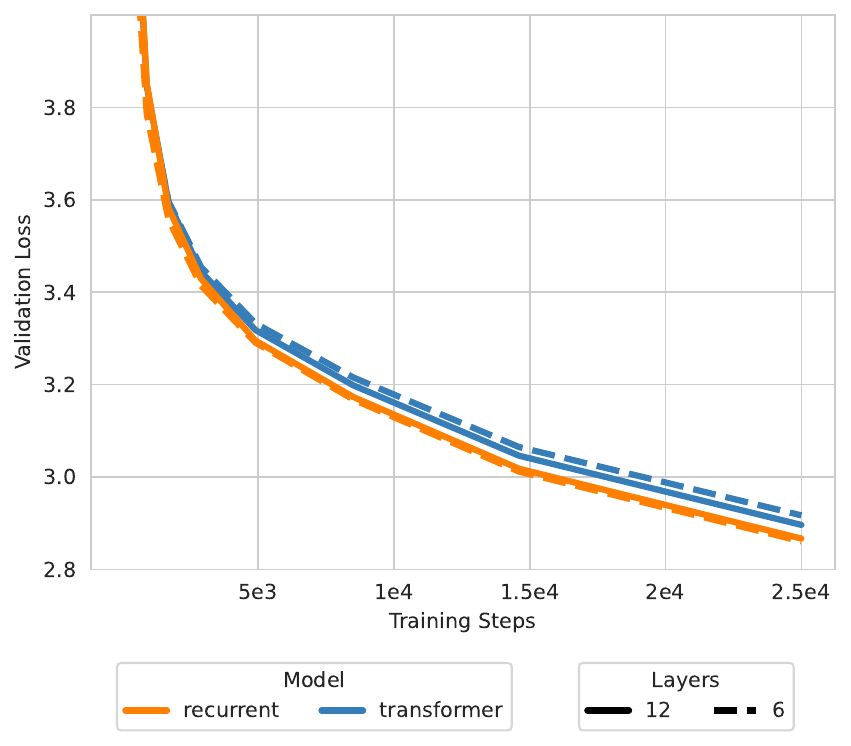}
  \caption{C4 pretraining: loss curves for 300m parameter model trained on C4 dataset.}
  \label{fig:c4-300m}
\end{minipage}
\hfill
\begin{minipage}[t]{0.45\linewidth}
\vspace{50pt}
  \centering
  \captionof{table}{C4 pretraining loss for 300M parameter model.}
  \label{tab:c4-512-300m}
  \begin{tabular}{lccc}
  \toprule
  Model & Layers & Width & Val CE $\downarrow$ \\
  \midrule
  Transformer & 6 & $2048$ & $2.917$ \\
  Transformer & 12 & $1408$ & $2.896$ \\
  Transformer & 24 & $1024$ & $2.892$ \\
  \methodname{} & 12 & $1408$ & $2.867$ \\
  \methodname{} & 6 & $2048$ & $2.86$ \\
  \bottomrule
  \end{tabular}
\end{minipage}

\end{figure}


\section{Architectural overview and notation}
\label{sec:arch}

\paragraph{Architectural overview.}
Relative to a standard causal Transformer, the defining change in \methodname{} is where the key--value pairs exposed to future positions come from. In a standard Transformer, the key--value pair at position $i$ is computed from the layer input at that position. In \methodname{}, by contrast, the \emph{persistent} key--value pair at position $i$ is computed from that position's layer output. Consequently, later positions attend to earlier positions whose representations have already undergone same-layer attention and MLP computation, making each layer recurrent along the temporal axis.

This creates a circularity at the current position: because the layer output at position $i$ also attends to the current position, the persistent pair $(\vk_i,\vv_i)$ cannot itself be used while computing that output. To resolve this, \methodname{} distinguishes between two kinds of key--value pairs. A \emph{temporary} pair, computed from the current layer input, is used only when evaluating attention at the current position. A \emph{persistent} pair, computed from the resulting layer output, is then stored and made available to all later positions.

\paragraph{Notation.}
We present the single-head formulation; multihead attention applies the same construction independently per head and then uses the usual output projection.
We assume a sequence length of $N$ and use $L$ for the number of stacked layers.
Let $D$ be the embedding dimension and consider a single layer with inputs $\vx_1,\ldots,\vx_N\in\R^D$.
Let $\MLP:\R^D\to\R^D$ denote the MLP block and let $\RMS:\R^D\to\R^D$ denote Root Mean Square normalization \citep{RMSNorm}. While in practice we use learnable parameters, as far as presentation and analysis is concerned, we take $\RMS(x)=\sqrt{D} \cdot \vx / \ltwonorm{\vx}$.
We use (magenta) $\qkRMS$ to distinguish query/key normalization~\citep{dehghani2023scalingQKNorm}.

The attention operator
$\Attn:(\R^D\times\R^D)^{*}\times \R^D\to\R^D$
maps a sequence of key--value pairs $(\vk_1,\vv_1),\ldots,(\vk_\ell,\vv_\ell)$ and a query $\vq$ to
\begin{align*}
\Attn\big((\vk_1,\vv_1),\ldots,(\vk_\ell,\vv_\ell),\vq\big) = \sum_{i=1}^\ell {\vv_i \cdot
\frac{\exp(\innerp{\vk_i}{\vq})}
{\sum_{j=1}^\ell{\exp(\innerp{\vk_j}{\vq}})}
}
\end{align*}

We use projection matrices $Q,K,V\in\R^{D\times D}$ to compute queries, keys and values based off an embedding. Following standard Transformer parameterizations~\cite{bordelon2023depthwise,yang2023tensor}, we use pre-LN~\cite{xiong2020layerPreLN} and assume attention and MLP residual updates are initialized/parameterized with an appropriate $1/\sqrt{L}$ scale so chaining maps of the form
$\vx\mapsto \vx+ \frac{1}{\sqrt{L}} \{\Attn,\MLP\}(\RMS(\vx))$
is well-behaved.

\subsection{The Transformer layer}
\label{sec:baseline-layer}

We first recall a standard \emph{causal} decoder-only Transformer layer~\citep{vaswani2017attention}. Given inputs $\vx_1,\ldots,\vx_N\in\R^D$, position $i$ forms its query, key, and value from the current layer input:
\begin{align*}
\vq_i &= \qkRMS[Q\,\RMS(\vx_i)], \\
\vk_i &= \qkRMS[K\,\RMS(\vx_i)], \\
\vv_i &= V\,\RMS(\vx_i).
\end{align*}
The attention output at position $i$ is then computed by attending over the prefix of key--value pairs available up to that position:
\begin{align*}
\va_i &= \Attn\big((\vk_1,\vv_1),\ldots,(\vk_i,\vv_i),\vq_i\big).
\end{align*}
Finally, the layer output is obtained by adding the attention and MLP residual branches:
\begin{align*}
\vy_i &= \vx_i + \frac{1}{\sqrt{L}}\left(\va_i + \MLP[\RMS(\vx_i+\frac{1}{\sqrt{L}}\va_i)]\right).
\end{align*}

The key structural point is that, in a standard Transformer, the key--value pair stored at position $i$ is computed from the layer input at the same position.


\subsection{The \methodname{} layer}
\label{sec:rt-layer}

\methodname{} layers (illustrated in Figure~\ref{fig:rt-arch}) differ from standard Transformer layers only in how the key--value pairs exposed to future positions are formed. At position $i$, \methodname{} first forms the query together with a \emph{temporary} key--value pair from the current layer input:
\begin{align*}
\vq_i &= \qkRMS[Q\,\RMS(\vx_i)], \\
\vk_i^{\temp} &= \qkRMS[K\,\RMS(\vx_i)], \\
\vv_i^{\temp} &= V\,\RMS(\vx_i).
\end{align*}
These definitions are identical to the Transformer's query, key, and value projections at position $i$. The attention output at position $i$ is then computed using the persistent key--value pairs from earlier positions together with the temporary pair at the current position:
\begin{align*}
\va_i
&=
\Attn\big(
(\vk_1,\vv_1),\ldots,(\vk_{i-1},\vv_{i-1}),
(\vk_i^{\temp},\vv_i^{\temp}),
\vq_i
\big).
\end{align*}
We next form the layer output representation
\begin{align*}
\vz_i &= \vx_i + \frac{1}{\sqrt{L}}\left(\va_i + \MLP[\RMS(\vx_i+\frac{1}{\sqrt{L}}\va_i)]\right),
\end{align*}
which is both the representation passed to the next layer and the source from which the persistent key--value pair at position $i$ is computed. We define that persistent pair by projecting from this output:
\begin{align}
\label{eq:persistentK}
\vk_i &= \qkRMS[K\,\RMS(\vz_i)], \\
\label{eq:persistentV}
\vv_i &= V\,\RMS(\vz_i).
\end{align}

Thus, $\vk_i^{\temp}, \vv_i^{\temp}$ is used only to compute attention at position $i$; it is not exposed to future positions. The persistent pair, by contrast, is defined only after $\vz_i$ has been formed and is then stored for use by all later positions. Thus, unlike in a standard Transformer, future positions attend not to a pair computed from the layer input at position $i$, but to one computed from the already-updated representation $\vz_i$. 

We reuse the same projection matrices $K$ and $V$ for both the temporary and persistent key--value pairs. Consequently, \methodname{} does not introduce additional key/value projection parameters relative to a Transformer; this reuse also preserves a shared semantics between the temporary and persistent key--value representations.

\subsection{Closest Related Work}
\label{sec:closest-related-work}
The closest representational relatives are Feedback Transformer variants.
Feedback Transformer~\citep{fan2020feedback} uses a cross-layer feedback memory shared across depth, essentially having just one list of key-value pairs computed based on the whole model's output rather than independently at each layer. Staircase Attention~\citep{ju2021staircase} generalizes Feedback Transformers, studying recurrent processing and caching variants with weight sharing -- still at a model rather than layerwise level.
This separation matters not only representationally but also computationally: within an \methodnameabv{} layer, all queries are available early, which is the enabling condition behind our efficient training methodology (Section~\ref{sec:tiling}).

TransformerFAM~\citep{hwang2024transformerfam} is closer in that it also operates independently at each layer and allows later positions to access more processed representations. However, it does so through a bounded memory that is read from and written to via attention. By contrast, \methodname{} retains per-token persistent key--value memory rather than compressing past information into a fixed-size state. This difference is important both for avoiding a bounded-memory bottleneck and for the representational results of Section~\ref{sec:repr-transformers}.

\section{Representational Perspective} \label{sec:repr}
In this section, we theoretically show that \methodname{} can emulate both a Transformer and RNN under mild assumptions. This shows that it subsumes both RNNs and Transformers, at least in the representation power.

\subsection{Representing Transformers}
\label{sec:repr-transformers}
Intuitively, \methodnames{} can recover the behavior of a standard Transformer of lower width by ensuring that the persistent key--value pairs computed from $\vz_i$ track those that would have been computed from $\vx_i$ via $K$ and $V$ projections. We concretize this statement below:

\begin{thm_samy}[informal]
Any width-$d'$ Transformer can be approximately simulated by a width-$d=3d'$ \methodname{}: the simulated Transformer activations can be embedded into disjoint feature groups of the \methodnameabv's embeddings. The \methodnameabv{} layer can be parameterized so that (i) attention scores are preserved and (ii) the layer output exactly tracks the Transformer layer output.
\label{thm:infrml-gen}
\end{thm_samy}

At a high level, the construction relies on representing smaller Transformer states inside a larger embedding of \methodnameabv{} by dedicating disjoint feature blocks to different roles. One block stores a protected copy of the layer input $\vx_i$ so that when \methodnameabv{} later computes persistent keys and values from the layer output $\vz_i$, it can still recover exactly the same key--value pairs the Transformer would have computed from $\vx_i$. A separate block is used to hold the attention contribution $\va_i$, so that when adding it to $\vx_i$ prior to applying the MLP, the contents of $\vx_i$ are protected from being lost. In this way, later tokens see identical attention scores, while the layer output matches the Transformer's layer output in the designated block. The width overhead of a factor of $3$, rather than $2$, is a subtle technical requirement for stacking multiple layers; a single layer can be replicated with an overhead of $2$. The complete construction and formal proof are provided in Appendix~\ref{app:transformer-sim}.

\subsection{Representing token-to-token recurrence}
\label{sec:repr-rnn}
If, using positional embeddings or biases, attention concentrates locally to the previous position, \methodnames{} implement an RNN-like update.
Formally, if $\va_i$ is dominated by the previous persistent value $\vv_{i-1}$, i.e. $\innerp{\vq_i}{\vk_{i-1}} \gg \innerp{\vq_i}{\vk^{\temp}_i}$ and $\innerp{\vq_i}{\vk_{i-1}} \gg \innerp{\vq_i}{\vk_j}$ for any $j < i - 1$, then we get that:
\begin{align*}
\vz_i &\approx \vx_i + \vv_{i-1} + \MLP[\RMS(\vx_i+\vv_{i-1})] = V\,\RMS(\vz_{i-1}) + \vx_i + \MLP[\RMS(\vx_i+V\,\RMS(\vz_{i-1}))]
\end{align*}
Under the additional simplifying assumption that $V$ is the identity, this becomes a particular state recurrence with a skip connection:
\begin{align*}
\vz_i = \RMS(\vz_{i-1}) + \vx_i + \MLP[\RMS(\vx_i+\RMS(\vz_{i-1}))]
\end{align*}
We do not claim to reproduce gated RNN/LSTMs, nor that training would yield to learning such structures.
We stress that representationally, \methodname{} is rich enough to express explicit iterative computation within a layer while also retaining full-prefix per-token memory (which was required to simulate Transformers in the previous section).

Crucially, once an architecture can represent such iterative computation, a natural question is whether the classic learnability issues of RNNs \citep{bengio1994long,pascanu2013difficulty} impede training.
Section~\ref{sec:paths} explains why \methodnames{} multi-hop dynamics can still train stably.

\subsection{Why temporal depth matters}
Transformers are shallow-through-time: deeper iterative computation along the sequence must be simulated primarily by stacking layers.
Theory on low-depth attention models and finite-automata tracking problems suggests that bounded depth can have concrete consequences, with shallow Transformers being representationally insufficient to simulate certain automata \citep{liu2022shortcuts} and more generally bound to TC0 - a class of shallow circuits~\citep{merrill2022saturated}.
\methodnames{} expose additional temporal depth within each layer. This is complementary to the depth obtained from stacking layers, thus pointing to the potential of achieving matching Transformers' effective depth while using fewer layers. We corroborate this hypothesis empirically in Section~\ref{sec:expts}.

\section{Training Stability of \methodname{}}
\label{sec:paths}
In this section, we explain how \methodname{} manages to avoid degenerate dynamics such as gradient vanishing or exploding through depth. We formalize our arguments by viewing the model as a directed computation graph over positions: there is an edge $i\to j$ when the computation at position $j$ \emph{directly} depends on quantities computed at position $i$.
In a classical RNN, information (and gradients) from position $i$ to $j$ must traverse the full chain
$i\to i\!+\!1\to\cdots\to j$,
and repeated composition along long chains leads to vanishing/exploding gradient phenomena \citep{bengio1994long,pascanu2013difficulty}.
This chain topology forces all influence from position $i$ to $j$ through $(j-i)$ successive state transitions. Stabilizing training typically requires these transitions to be close to contractive, but then the influence of $\vx_i$ on $\vx_j$ shrinks rapidly with $(j-i)$, making distant information difficult to transmit. While in RNNs this issue can be alleviated through careful initialization schemes~\citep{orvieto2023resurrecting}, our method takes advantage of the fact that there are both direct hops, as well as additional multi-hop paths between layers.

As in a standard Transformer, token $j$ can directly attend to any earlier token $i<j$ via the stored key--value pair $(\vk_i,\vv_i)$, creating a one-hop information path $i\!\to\! j$.
The key difference is that in \methodname{}, the stored pair $(\vk_i,\vv_i)$ is computed from the \emph{layer output} $\vz_i$, and $\vz_i$ already includes the result of attending to earlier stored pairs.
Consequently, information can propagate not only directly from $i$ to $j$, but also \emph{indirectly}.

Concretely, a multi-hop path from token $1$ to token $4$ (Figure~\ref{fig:rt-arch}) can go through intermediate write--read steps:
\[
\vx_1 \to \vz_1 \to (\vk_1,\vv_1) \to \va_2 \to \vz_2 \to (\vk_2,\vv_2) \to \va_4 \to \vz_4
\]
Here each step $\vz_t \to (\vk_t,\vv_t)$ is a \emph{write} to the layer's persistent memory, and each step $(\vk_t,\vv_t) \to \va_{t'}$ is a \emph{read} by attention at any later position $t'>t$.
Chaining these write--read operations yields multi-hop influence paths whose length scales with the distance between positions, enabling within-layer iterative computation (as in Section~\ref{sec:repr-rnn}) while preserving the direct one-hop attention routes of a Transformer.
\paragraph{Dampening long paths without eliminating long-range access.}
Multi-hop paths are only useful if they do not explode.
In practice, two levers dominate.
First, standard depth-wise scaling conventions for residual branches keep per-layer updates in a stable range \citep{bordelon2023depthwise,yang2023tensor}.
Second, normalization preceding computation of persistent keys/values (the $\RMS(\vz_i)$ inside Equations~\ref{eq:persistentK}-\ref{eq:persistentV}) controls magnitudes even though $\vz_i$ is a sum of multiple components.
Empirically, these choices place long multi-hop influences on the vanishing end: longer chains have smaller effect.
Unlike a pure RNN, this does not remove long-range access because direct attention edges remain available even when long chains are damped.

In the theorem below, for a very simplified setup without normalization, we show that we do not get exploding gradients at initialization. Normalization helps in stability, that is, it allows stable training of \methodname{}  with higher learning rates. This is demonstrated empirically in Appendix \ref{app:layernorm_stable}.

\begin{theorem} \label{thm:train_stable}
    Consider a simplified 1-layer uniform-attention only \methodnameabv{} layer with inputs given by $x_1,...,x_n$ and outputs denoted as $z_1,...,z_n$, where
    \[ z_k = x_k + \frac{\alpha}{k} \left( V x_k + V \sum_{j=1}^{k-1} z_j \right) \]
    where $\alpha$ is a scalar denoting the scaling of the residual and $V$ is the value matrix. Then, for $k \geq 2$,

    \[ \frac{\partial z_k}{\partial x_1} = \frac{1}{k!} \sum_{r=1}^k {k \brack r}\,\alpha^r V^r \]

    where ${k \brack r}$ denotes the total number of permutations of k elements having exactly $r$ cycles.
\end{theorem}

As the total number of permutations is $k!$, the theorem above shows that as long as the maximum eigenvalue of $\alpha V$ is smaller than $1$, we do not get an exploding gradient from $z_j$ to $x_1$. Thus, for orthonormal initialization, for any $\alpha < 1$, we expect to be in this regime. Moreover, since the overall gradient is summed over paths of various lengths (given by $r$ in the above expression), we can see that we have non-vanishing gradient even when the maximum eigenvalue of $\alpha V$ is smaller than $1$. In particular, since ${k \brack 1} = (k-1)!$, the term in the above expression corresponding to $r=1$ is $\frac{\alpha}{k} \cdot V$, which is precisely the gradient a vanilla transformer would yield. Proof for this theorem can be found in Appendix \ref{app:train_stable}.

\begin{figure}[ht]
  \begin{center}
    \centerline{\includegraphics[width=0.6\columnwidth]{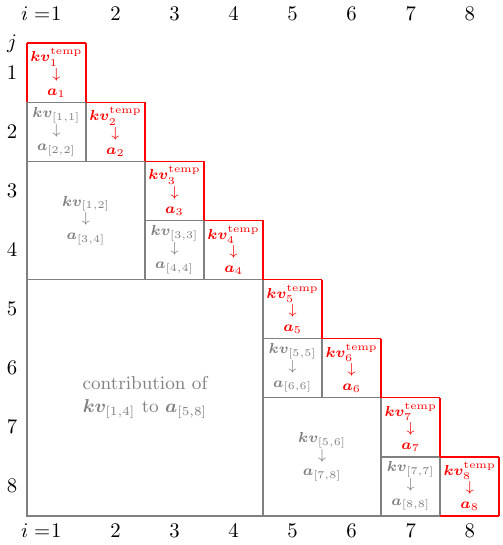}}
    \caption{
      We use the tiling of~\cite{flashInference} to increase arithmetic intensity during the forward pass since $(\vk_t, \vv_t)$ only become available after the attention output $\va_t$ is computed - this in turn happens once position $t$ has attended to all previous KVs.
    }
    \label{fig:tiling}
  \end{center}
\end{figure}

\section{Exact Tiling for Training and Prefill}
\label{sec:tiling}

\paragraph{What makes naive evaluation slow.}
During training/prefill, \methodnames{} are fundamentally sequential in position;
to compute the persistent pair $(\vk_t,\vv_t)$ we must first compute $\vz_t$, and $\vz_t$ depends on $\va_t$, which aggregates over all previous persistent key--value pairs: $\{(\vk_i, \vv_i)\}_{i<t}$. Therefore, a naive implementation reveals persistent keys/values one position at a time by having each new query aggregate over a growing prefix, yielding low reuse and high memory traffic.

\paragraph{A short Roofline view: why we care about arithmetic intensity.}
The Roofline model \citep{williams2009roofline} bounds attainable throughput by either peak compute or peak memory bandwidth depending on arithmetic intensity (FLOPs per byte moved).
When attention repeatedly streams large prefixes of keys/values to produce small incremental updates, effective arithmetic intensity (AI) can be close to constant, making the operation bandwidth-bound even on large accelerators.
This is the regime where reorganizing memory movement (even without changing the math) can give large wins.

\paragraph{Enabling observation: within-layer queries are available early.}
Despite the sequential reveal of persistent keys/values, all queries $\{\vq_i\}_{i=1}^N$ in a layer depend only on the layer input $\{\vx_i\}$ and can be computed early on in parallel. This means that one could do some eager work of "aggregating" the contribution of any key--value pairs available thus far, to any future queries, not just the immediately upcoming one. For example, after $(\vk_4, \vv_4)$ are computed, naively, we would wait until next step when we need to know $\va_5$; for that, we check the whole prefix of $4$ key--value pairs, "inquiring" about \emph{just one} query ($\vq_5$). This "just one" is what gives the arithmetic intensity of $\approx 1$. Alternatively, one can already start accounting for how they contribute to $\va_5\ldots \va_8$ - inquiring about $4$ queries at once and thus raising the arithmetic intensity to $\approx 4$ \footnote{Since we also need to load the queries, for $\text{cnt}_q$ queries to attend to $\text{cnt}_{kv}$ key--value pairs we get an AI of $\Theta(2\cdot \text{cnt}_q \cdot \text{cnt}_{kv} / (\text{cnt}_q + 2\text{cnt}_{kv}))$}.


A very similar regime is exploited in the Flash Inference framework \citep{flashInference} -- while their framework is meant for decoding, one forward pass of \methodnameabv{} is essentially a sequence of decode steps. It applies to our case since the computation of interest is:
\begin{itemize}
    \item Contribution-based ($\va_i$ can be accumulated over different groups of key--value pairs)
    \item the contribution is independent of future $\vz$'s (i.e., all queries are readily available)
\end{itemize} 
The second condition also clarifies why the same approach cannot extend to cross-layer feedback architectures~\citep{fan2020feedback,ju2021staircase}: when future queries indirectly depend on feedback that is only produced after running later layers, queries are not all available early.

\paragraph{Exact tiling schedule.}
Our algorithm is an exact evaluation algorithm: it computes the same attention outputs up to floating-point reordering effects.
The schedule follows the tiling in Figure~\ref{fig:tiling}.
It interleaves
\begin{itemize}
\item computing $\vz_t$ (via $\MLP$) as soon as $\va_t$ is available, to then reveal the new persistent key--value pair ($\vk_t, \vv_t)$ and
\item updating attention accumulators for several future queries that are already known - by processing the newly freed tile.
\end{itemize}
For example, as $\va_6$ becomes available, $\vz_6$ and then ($\vk_6, \vv_6)$ are computed and then one can process all the the contribution of $\{(\vk_5, \vv_5), (\vk_6, \vv_6)\}$ to $\{\va_7, \va_8\}$ (by "asking" queries $\vq_7, \vq_8$).
In order to aggregate attention contribution, we maintain the same online softmax statistics as \citep{rabe2021blockAttention,dao2022flashattention} (running attention score maxima and normalizing factor) so that contributions from multiple key/value tiles can be accumulated stably. The full algorithm description is available in Appendix~\ref{app:complete-algo}.

\begin{figure}[t]
    \centering
    \includegraphics[width=0.6\linewidth]{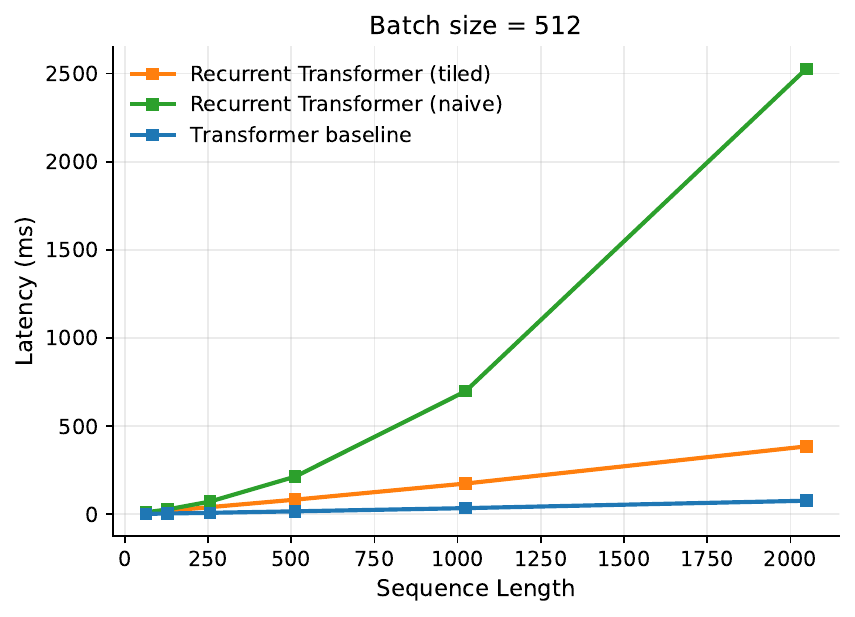}
    \caption{One-layer forward-pass latency as a function of sequence length at batch size $512$ on a single H100 GPU with $1024$ width. The naive recurrent implementation shows approximately quadratic growth with context length, whereas the tiled implementation scales much closer to linearly. This matches the intended effect of the tiled schedule, which increases reuse of loaded key--value pairs across multiple future queries. We also include the vanilla Transformer baseline for reference.}
    \label{fig:latency-scaling}
\end{figure}

\paragraph{Asymptotics.}
Counting HBM movement, the naive one-query-at-a-time implementation incurs $\Theta(N^2)$ memory traffic.
The tiled schedule reduces traffic to $\tilde{\Theta}(N\log N)$ by reusing streamed key/value tiles across many queries, while attention FLOPs remain $\Theta(N^2)$.
Consequently, effective arithmetic intensity increases from $\Theta(1)$ to $\Theta(N/\log N)$. The gains of this tiling approach can be seen in Figure~\ref{fig:latency-scaling} - while the latency of a naive eager implementation grows approximately quadratically with context length, our method exhibits near-linear scaling. 

\section{A deep dive into the computational challenges} \label{sec:comp_challenges}

In this section, we outline the key computational design decisions required to make \methodname{} training practically efficient, enabling the execution of our language modeling experiments. In contrast to the tiling algorithm discussed earlier, which focuses on algorithmic structure, the emphasis here is on implementation-level optimizations. These changes do not alter asymptotic complexity, but instead yield meaningful constant-factor speedups that are critical for reducing overall training time. 

\paragraph{The setup.}
We run all experiments on H100 GPUs, with training carried out on a single device at a time. Our implementation is in PyTorch~\citep{paszke2017automaticPyTorch}. Beyond the algorithmic tiling strategy of Section~\ref{sec:tiling}, we make several implementation choices aimed at improving hardware utilization and reducing overheads. While we successfully use Torch compile to fuse a number of components, we do not rely on custom kernels in the current implementation. We leave such kernel-level optimization to future work and focus here on the algorithmic improvements introduced by the architecture and computation schedule.

Unless otherwise noted, all latency measurements assume hidden dimension $1024$ and $16$ attention heads, are averaged over $5$ runs after $3$ warmup runs, and have standard deviation below $1$ ms. Latencies are reported on a per-layer basis: they measure the computation of a single layer’s map $\vx \mapsto \vz$, including both attention and MLP computation, but excluding embedding, unembedding, and loss computation, which are identical across Recurrent Transformers and Transformers.

\subsection{MLPs and batch size}
\label{sec:batch-size}

While the tiling algorithm greatly improves the arithmetic intensity of the attention component of~\methodnames{}, the MLP computations must be interleaved with it and can become the dominant cost of a forward pass. The reason is that, unlike in a standard Transformer, the MLP does not receive all $B \times N$ tokens at once. Instead, it processes only $B$ tokens at a time over $N$ iterations, one position at a time. As a result, the per-device batch size $B$ directly controls the arithmetic intensity of the MLP; in the regime $B \le O(d)$, this intensity is approximately linear in $B$.

In practice, for the model scales we consider, $B = 512$ provides a favorable trade-off between GPU utilization and activation memory. This constraint is fundamental to recurrent-in-time architectures and is not specific to the attention mechanism; similar issues arise in classical recurrent models such as LSTMs. In particular, sustaining such a batch size on a single device is already challenging from the perspective of activation memory, even for models in the 150M--300M parameter range. Under ordinary circumstances, one would employ gradient accumulation, but that defeats the purpose here: increasing the total batch size via accumulation does not increase the effective batching seen by the MLP, and therefore does not improve arithmetic intensity.

\paragraph{Activation Checkpointing} We therefore rely on activation checkpointing. One useful property of our computation is that once the inputs $\vx$ to each layer are stored, the outputs $\vz$ can be retained essentially ``for free,'' since they serve as the inputs to the next layer. This leads to a substantially cheaper recomputation procedure. In particular, once the persistent key--value pairs have been (parallelly) reconstructed from $\vz$, the remaining intermediate quantities can be recovered in a fully parallel manner. In particular, the attention-related intermediates can be recomputed without replaying the slow sequential process by which the $\vz_i$ were originally revealed one at a time. Consequently, although we still incur the standard cost of checkpointing, the recomputation overhead is meaningfully smaller than that of the original forward pass.

\paragraph{Critical batch size.}
Even if memory permitted arbitrarily large batches, there remains a statistical efficiency limit to how many tokens can be processed per optimizer step before optimization quality begins to degrade~\citep{McCandlish2018Empirical, Shallue2018Measuring}. In our setup, this critical batch size is around $256$K tokens per optimizer iteration for Transformer models in the 150M and 300M parameter range~\citep{zhang2025how}. Accordingly, throughout our experiments we use sequences of length $512$ and batch size $512$, corresponding to $256$K tokens per iteration. In Appendix~\ref{app:150m-pretrain}, we further verify that the critical batch size of~\methodname{} is not below this value.

\subsection{Using CUDA Graphs}

\begin{table}
\centering
\caption{One layer of Recurrent Transformer forward pass latency (ms) for sequences of $512$ tokens each. CPU overhead dominates at lower batch sizes and we employ CUDA Graphs to mitigate this.}
\label{tab:cudagraph}
\begin{tabular}{lrr}
\toprule
batch size & without CUDA Graphs & with CUDA Graphs \\
\midrule
32 & 277.08 & 38.85 \\
64 & 279.23 & 42.28 \\
128 & 279.42 & 48.95 \\
256 & 276.05 & 61.39 \\
512 & 277.33 & 81.73 \\
1024 & 275.49 & 134.09 \\
2048 & 293.70 & 240.83 \\
\bottomrule
\end{tabular}
\end{table}

Both the attention and MLP portions of our architecture involve $O(N)$ launches of moderately small kernels. Even when the underlying kernels are themselves compute-bound, the amount of work per kernel can be small enough that CPU-side dispatch overhead becomes a dominant bottleneck. Ordinarily, launch overhead is hidden because future kernels can be enqueued while current ones are still executing. Here, however, the kernels are sufficiently short-lived that this overlap is limited, and dispatch latency becomes visible on the critical path.

For this reason, we use CUDA Graphs, recording the full forward (and backward) pass computation and replaying it with a single launch. This turns out to be crucial for performance. The resulting latency improvements are reported in Table~\ref{tab:cudagraph}. One noteworthy feature of Table~\ref{tab:cudagraph} is that, without CUDA Graphs, latency remains nearly flat across a wide range of batch sizes, indicating that dispatch overhead rather than arithmetic work is the main bottleneck. With CUDA Graphs enabled, latency scales much more meaningfully with batch size, reflecting the underlying compute cost more faithfully.

\subsection{Cache locality and memory access pattern}

The tiling schedule also has a favorable memory-access pattern. As we iterate through positions, the portions of the KV cache accessed at successive steps overlap heavily and are often quite small: on average involving only $O(\log N)$ positions. To exploit this locality, we store the KV cache with the position dimension first, rather than the more conventional batch- or head-major layout. This ensures that the slices accessed by each tiled update are contiguous in memory, improving cache locality and reducing unnecessary memory movement.

\subsection{The backward pass}

To avoid repeated \texttt{cat}-operations, which would increase both memory traffic and peak memory usage, we preallocate the persistent KV cache and write to it in place. Since PyTorch autograd is not compatible with in-place operations, we implement a custom backward pass.

A naive implementation would simply mimic the reverse traversal that autograd would have carried out on the corresponding computational graph. However, the structure of the computation allows a more parallel schedule. In particular, within the reverse loop over positions, one only needs to accumulate the gradients with respect to $(\vk_i, \vv_i)$. The reason is that, before moving from position $i$ to position $i-1$ and propagating the effect of $\va_{i-1}$ onto earlier key--value pairs, one must already have the final gradient with respect to $\vz_{i-1}$, which itself depends in part on $(\vk_{i-1}, \vv_{i-1})$. By contrast, gradients with respect to $\vx$, $\vq$, the temporary key--value pairs, and the model parameters do not impose such immediate dependencies and can therefore be computed outside the loop in a batched manner via larger kernels. This substantially improves parallelism in the backward pass.



\section{Experiments}
\label{sec:expts}

We evaluate \methodname{} on synthetic tasks designed to stretch models' representation ability, as well as language modeling.

\subsection{Synthetic diagnostics}
We use the MAD suite~\citep{poli2024mad} as a diagnostic for hybrid architectures.
We also include the copy task~\citep{jelassi2024repeat} that is provably impossible to solve via models of finite memory (this classification includes all forms of RNNs, including SSMs~\cite{gu2024mamba}).
These diagnostics are not intended as long-range benchmarks; they isolate whether recurrence is being used in the intended way. Since we want to measure the effective depth of a layer, we compare one Transformer layer to one layer of \methodnameabv{}, otherwise preserving the same model configurations as in~\citep{poli2024mad}. The precise hyperparameter details are provided in Appendix \ref{app:hyper_synth}. The sequence-level accuracies are displayed in Figure~\ref{fig:synth} and show \methodnameabv s significantly outperform Transformers that do not achieve meaningful performance on any of the tasks.

\begin{figure*}[t]
  \centering
  \includegraphics[width=\textwidth]{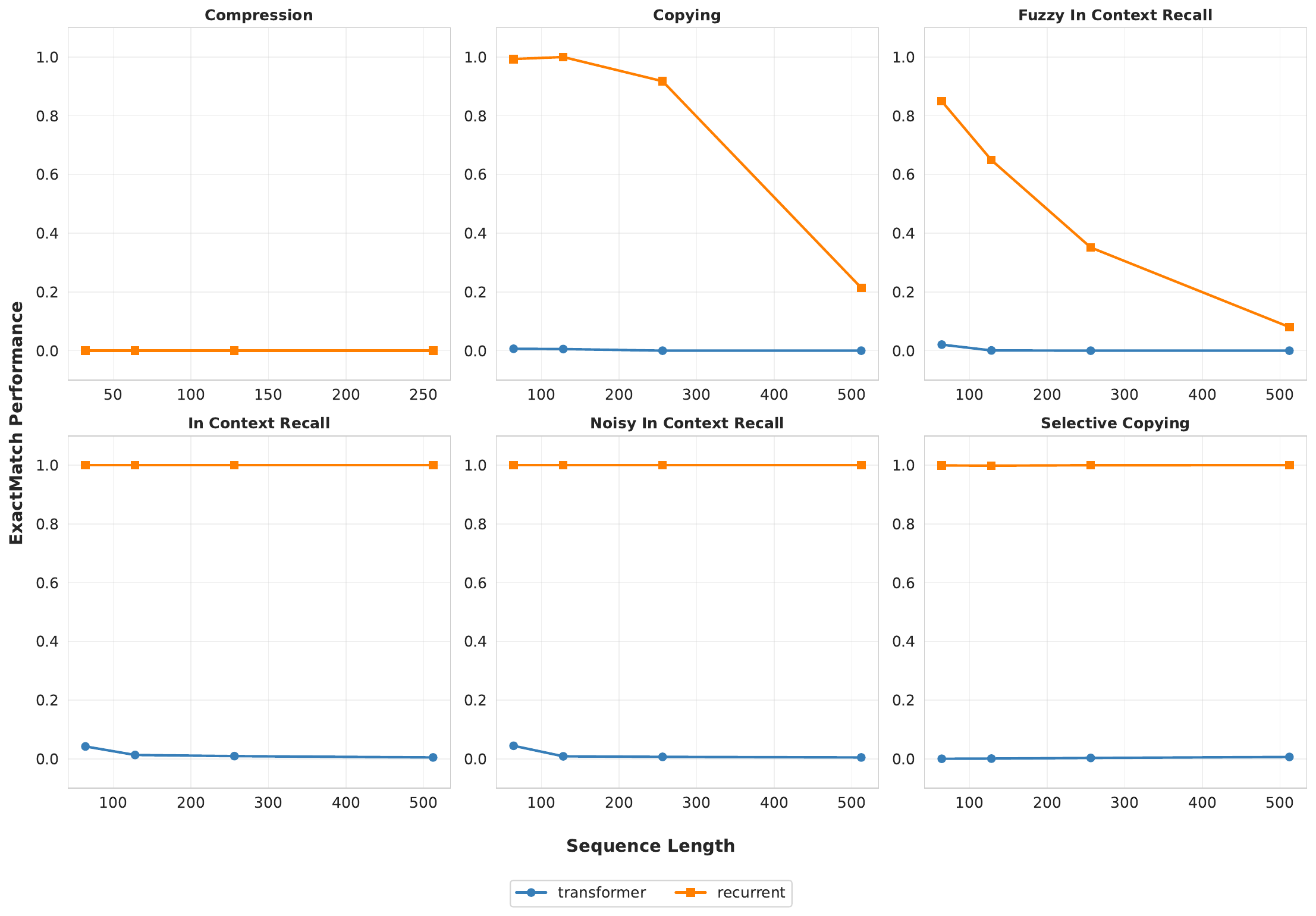}
  \caption{Sequence-level accuracy of the \methodname{} and a regular Transformer on MAD synthetic tasks and the copy task. \methodnameabv outperforms Transformers in all tasks but compression. Neither model achieves non-trivial performance on compression at sequence-level, but they do achieve meaningful token-level accuracy with \methodnameabv{} still in the lead, as shown in Appendix~\ref{app:synthetics-token-level}.}
  \label{fig:synth}
\end{figure*}

\begin{table}[t]
\centering
\caption{Downstream performance for the 300M model.}
\label{tab:c4-512-300m-downstream}
\begin{tabular}{lccccccc}
\toprule
Model & Layers & piqa CE & hellaswag CE & arc easy CE & openbook qa CE & sciq CE & winogrande CE  \\
\midrule
Transformer & 6 & $5.536$ & $4.334$ & $11.128$ & $12.408$ & $14.514$ & $8.413$ \\
Transformer & 12 & $5.508$ & $4.156$ & $10.894$ & $12.36$ & $14.15$ & $7.809$ \\
Recurrent & 12 & $\textbf{5.276}$ & $\textbf{4.052}$ & $\textbf{10.336}$ & $\textbf{11.557}$ & $13.384$ & $\textbf{7.628}$ \\
Recurrent & 6 & $5.356$ & $4.122$ & $10.388$ & $11.65$ & $\textbf{13.206}$ & $7.914$ \\
\bottomrule
\end{tabular}
\end{table}

\subsection{Language modeling on C4 (300M parameters)}
We implement the \methodname{} on top of the OLMo-2~\cite{olmo20242} codebase and pretrain 300M and 150M parameter models on C4~\citep{C4raffel2020exploring}, for $1\times$ Chinchilla tokens ($\approx 3b$ tokens). The precise hyperparameters are provided in Appendix \ref{app:hyper_c4}. Figure~\ref{fig:c4-300m} shows the cross-entropy loss for the 300M model during training for Transformers and \methodnames{} at $12$ layers ($24$ layers is the standard configuration used in previous works such as \citet{zhao2025deconstructing} which performs comparably as shown in Table~\ref{tab:c4-512-300m}) and parameter-equivalent $6$ layers ($d$ scaled up by $\sqrt{2}$ from $1408$ to $2048$). Table~\ref{tab:c4-512-300m} contains the final cross entropy values. As shown, \methodnameabv s outperform Transformers meaningfully by a delta cross-entropy of $0.03$ at $12$ layers and $0.057$ at $6$ layers.
Notably, the layerwise recurrence shifts the depth--width optimum at fixed parameter count. 

We evaluate the model's downstream performance on 6 multiple choice tasks used in OLMo \citep{groeneveld2024olmoacceleratingsciencelanguage}: PIQA \citep{Bisk2020}, Hellaswag \citep{Zellers2019}, ARC Easy \citep{Clark2018ARC}, OpenBookQA \citep{Mihaylov2018}, SciQ \citep{Welbl2017} and Winogrande \citep{Sakaguchi2020}.
We provide the model's CE loss values on the ground truth answers in Table~\ref{tab:c4-512-300m-downstream}. We used this metric as it is known to be smoother at small scales \citep{bhagia2025establishingtaskscalinglaws}. We also provide the downstream accuracies in Appendix Table~\ref{tab:c4-512-300m-downstream-acc}, but most of them are close to random noise (with Winogrande being close to 50\%, while most others being 6-7\% above random). The results for the 150M model are provided in Appendix~\ref{app:150m-pretrain}.

\subsection{Depth--width tradeoffs and decoding footprint}
For autoregressive decoding, \methodname{} exhibits essentially the same per-token attention behavior as a Transformer with comparable depth and width: each new token attends to cached keys and values computed from preceding tokens.
If, however, \methodname{} achieves comparable model quality using fewer layers—reduced by a factor of $\alpha$—while keeping the total parameter count fixed, the size of the key–value (KV) cache decreases by a factor of $\sqrt{\alpha}$. This follows from the corresponding increase in model width by only $\sqrt{\alpha}$.
A smaller KV cache directly reduces memory traffic during decoding, which can lead to higher throughput in bandwidth-limited settings. More broadly, trading off depth for width may be advantageous for decoding, since increased width can be more effectively parallelized using techniques such as tensor parallelism. We leave a detailed evaluation of fully optimized decoding latency to future work.


\section{Related Work}
\label{sec:related}

We group the most relevant prior work by the core bottleneck it imposes and by whether it introduces recurrence/feedback-like computation beyond standard feedforward self-attention.

\paragraph{Bounded-memory sequence models.}
Classical RNNs and modern state-space models maintain a fixed-size state that is updated recurrently \citep{bengio1994long,pascanu2013difficulty,LSTMhochreiter1997long,smith2022simplified,gu2024mamba}. Linear-attention/retention variants also admit recurrent formulations with bounded state \citep{katharopoulos2020transformers,sun2023retentive,peng2023rwkv}. While computationally attractive, bounded-state families cannot in general preserve information that scales with sequence length; \citep{jelassi2024repeat} highlight this limitation by proving such models cannot perform the copy task. This limitation is in contrast to \methodnameabv (as shown in Section~\ref{sec:repr-transformers}).

\paragraph{Segment recurrence and memory mechanisms.}
Transformer-XL and follow-ups process context in segments and potentially summarize them via attention mechanisms~\citep{dai2019transformerxl,rae1911compressive}. RMTs take this a step forward by summarizing the feedback information across segments~\citep{bulatov2022rmt}. These methods primarily address efficient long-context handling rather than layerwise recurrent computation over token states. Furthermore, they have the same limitation that classic RNNs have - namely bounded memory.

\paragraph{Recurrent/feedback Transformers.}
Feedback Transformers~\citep{fan2020feedback}, Staircase Attention~\citep{ju2021staircase}, and TransformerFAM~\citep{hwang2024transformerfam} introduce recurrent/feedback-style computation in Transformer blocks. As discussed in Section~\ref{sec:closest-related-work}, \methodname{} is \emph{layerwise} recurrent with separate per-layer key--value collections (rather than cross-layer shared feedback), and this structure is also what enables our efficient training/prefill (Section~\ref{sec:tiling}).

\section{Discussion and Conclusion}
\label{sec:conclusion}
In this work, we introduce \methodname{}, which integrates three key ideas: (i) a deeper-in-time representation within each layer, (ii) a path-based perspective that enables longer multi-hop influences while retaining direct attention access by operating closer to the vanishing-gradient regime, and (iii) an exact, I/O-aware evaluation algorithm that makes training and prefill practical by reducing memory traffic without altering the underlying computation. Our results clearly demonstrate the increased effective depth provided by introducing layerwise recurrence.

We view this work as a proof of concept that opens several directions for future research. As with any new architecture, the introduction of layerwise recurrence may alter tuning behavior and scaling laws, potentially shifting the optimal depth–width trade-off. In addition, the current design can be extended with blocking, in which recurrence is executed over blocks of steps, yielding a controllable trade-off between recurrent depth and training speed. Finally, while the proposed tiling algorithm is exact and delivers measurable gains, further improvements are likely achievable through fully optimized kernels and extensions of existing parallelization techniques, which we leave to future work.

In conclusion, layerwise recurrence provides a simple and principled mechanism for exposing additional temporal depth while retaining a memory that scales with sequence length and preserves the full representational capacity of the Transformer. When combined with an exact tiling strategy that enables computation reuse and reduces HBM traffic, \methodnames{} make recurrent-in-layer training and prefill feasible in practice and shift depth–width trade-offs at a fixed parameter count. This shift is also beneficial at decode time, where a reduced KV cache size leads to improved efficiency.



\section*{Acknowledgements}
DM, AM, MK acknowledge the support of a Kempner Institute Graduate Research Fellowship. The authors
acknowledge that this work has been made possible in part by a gift from the Chan Zuckerberg
Initiative Foundation to establish the Kempner Institute for the Study of Natural and Artificial
Intelligence. SK, CO and DM acknowledge support from the Office of Naval Research under award N0001422-1-2377 and the National Science Foundation Grant under award \#IIS 2229881. DM is also supported by a Simons Investigator Fellowship, NSF grant DMS-2134157, DARPA grant W911NF2010021,and DOE grant DE-SC0022199.

\bibliography{references}

\begin{thebibliography}{45}
\providecommand{\natexlab}[1]{#1}
\providecommand{\url}[1]{\texttt{#1}}
\expandafter\ifx\csname urlstyle\endcsname\relax
  \providecommand{\doi}[1]{doi: #1}\else
  \providecommand{\doi}{doi: \begingroup \urlstyle{rm}\Url}\fi

\bibitem[Bengio et~al.(1994)Bengio, Simard, and Frasconi]{bengio1994long}
Y.~Bengio, P.~Simard, and P.~Frasconi.
\newblock Learning long-term dependencies with gradient descent is difficult.
\newblock IEEE Transactions on Neural Networks, 1994.
\newblock Often cited via the 1994 journal version.

\bibitem[Bhagia et~al.(2025)Bhagia, Liu, Wettig, Heineman, Tafjord, Jha, Soldaini, Smith, Groeneveld, Koh, Dodge, and Hajishirzi]{bhagia2025establishingtaskscalinglaws}
A.~Bhagia, J.~Liu, A.~Wettig, D.~Heineman, O.~Tafjord, A.~H. Jha, L.~Soldaini, N.~A. Smith, D.~Groeneveld, P.~W. Koh, J.~Dodge, and H.~Hajishirzi.
\newblock Establishing task scaling laws via compute-efficient model ladders, 2025.
\newblock URL \url{https://arxiv.org/abs/2412.04403}.

\bibitem[Bisk et~al.(2020)Bisk, Zellers, Bras, Gao, and Choi]{Bisk2020}
Y.~Bisk, R.~Zellers, R.~L. Bras, J.~Gao, and Y.~Choi.
\newblock Piqa: Reasoning about physical commonsense in natural language.
\newblock In \emph{Proceedings of the AAAI Conference on Artificial Intelligence}, 2020.

\bibitem[Bordelon et~al.(2023)Bordelon, Noci, Li, Hanin, and Pehlevan]{bordelon2023depthwise}
B.~Bordelon, L.~Noci, M.~B. Li, B.~Hanin, and C.~Pehlevan.
\newblock Depthwise hyperparameter transfer in residual networks: Dynamics and scaling limit, 2023.

\bibitem[Bulatov et~al.(2022)Bulatov, Kuratov, and Burtsev]{bulatov2022rmt}
A.~Bulatov, Y.~Kuratov, and M.~Burtsev.
\newblock Recurrent memory transformer.
\newblock \emph{Advances in Neural Information Processing Systems}, 35:\penalty0 11079--11091, 2022.

\bibitem[Clark et~al.(2018)Clark, Cowhey, Etzioni, Khot, Sabharwal, Schoenick, and Tafjord]{Clark2018ARC}
P.~Clark, I.~Cowhey, O.~Etzioni, T.~Khot, A.~Sabharwal, C.~Schoenick, and O.~Tafjord.
\newblock Think you have solved question answering? try arc, the ai2 reasoning challenge.
\newblock In \emph{arXiv preprint arXiv:1803.05457}, 2018.

\bibitem[Dai et~al.(2019)Dai, Yang, Yang, et~al.]{dai2019transformerxl}
Z.~Dai, Z.~Yang, Y.~Yang, et~al.
\newblock Transformer-{XL}: Attentive language models beyond a fixed-length context.
\newblock \emph{arXiv preprint arXiv:1901.02860}, 2019.

\bibitem[Dao et~al.(2022)Dao, Fu, Ermon, Rudra, and R{\'e}]{dao2022flashattention}
T.~Dao, D.~Y. Fu, S.~Ermon, A.~Rudra, and C.~R{\'e}.
\newblock Flashattention: Fast and memory-efficient exact attention with io-awareness, 2022.

\bibitem[Dehghani et~al.(2023)Dehghani, Djolonga, Mustafa, Padlewski, Heek, Gilmer, Steiner, Caron, Geirhos, Alabdulmohsin, et~al.]{dehghani2023scalingQKNorm}
M.~Dehghani, J.~Djolonga, B.~Mustafa, P.~Padlewski, J.~Heek, J.~Gilmer, A.~P. Steiner, M.~Caron, R.~Geirhos, I.~Alabdulmohsin, et~al.
\newblock Scaling vision transformers to 22 billion parameters.
\newblock In \emph{International conference on machine learning}, pages 7480--7512. PMLR, 2023.

\bibitem[Fan et~al.(2020)Fan, Lavril, Grave, Joulin, and Sukhbaatar]{fan2020feedback}
A.~Fan, T.~Lavril, E.~Grave, A.~Joulin, and S.~Sukhbaatar.
\newblock Addressing some limitations of transformers with feedback memory.
\newblock \emph{arXiv preprint arXiv:2002.09402}, 2020.

\bibitem[Groeneveld et~al.(2024)Groeneveld, Beltagy, Walsh, Bhagia, Kinney, Tafjord, Jha, Ivison, Magnusson, Wang, Arora, Atkinson, Authur, Chandu, Cohan, Dumas, Elazar, Gu, Hessel, Khot, Merrill, Morrison, Muennighoff, Naik, Nam, Peters, Pyatkin, Ravichander, Schwenk, Shah, Smith, Strubell, Subramani, Wortsman, Dasigi, Lambert, Richardson, Zettlemoyer, Dodge, Lo, Soldaini, Smith, and Hajishirzi]{groeneveld2024olmoacceleratingsciencelanguage}
D.~Groeneveld, I.~Beltagy, P.~Walsh, A.~Bhagia, R.~Kinney, O.~Tafjord, A.~H. Jha, H.~Ivison, I.~Magnusson, Y.~Wang, S.~Arora, D.~Atkinson, R.~Authur, K.~R. Chandu, A.~Cohan, J.~Dumas, Y.~Elazar, Y.~Gu, J.~Hessel, T.~Khot, W.~Merrill, J.~Morrison, N.~Muennighoff, A.~Naik, C.~Nam, M.~E. Peters, V.~Pyatkin, A.~Ravichander, D.~Schwenk, S.~Shah, W.~Smith, E.~Strubell, N.~Subramani, M.~Wortsman, P.~Dasigi, N.~Lambert, K.~Richardson, L.~Zettlemoyer, J.~Dodge, K.~Lo, L.~Soldaini, N.~A. Smith, and H.~Hajishirzi.
\newblock Olmo: Accelerating the science of language models, 2024.
\newblock URL \url{https://arxiv.org/abs/2402.00838}.

\bibitem[Gu and Dao(2024)]{gu2024mamba}
A.~Gu and T.~Dao.
\newblock Mamba: Linear-time sequence modeling with selective state spaces.
\newblock In \emph{First conference on language modeling}, 2024.

\bibitem[Hochreiter and Schmidhuber(1997)]{LSTMhochreiter1997long}
S.~Hochreiter and J.~Schmidhuber.
\newblock Long short-term memory.
\newblock \emph{Neural computation}, 9\penalty0 (8):\penalty0 1735--1780, 1997.

\bibitem[Hwang et~al.(2024)Hwang, Wang, Huo, Sim, and Moreno~Mengibar]{hwang2024transformerfam}
D.~Hwang, W.~Wang, Z.~Huo, K.~C. Sim, and P.~Moreno~Mengibar.
\newblock Transformerfam: Feedback attention is working memory.
\newblock \emph{arXiv preprint arXiv:2404.09173}, 2024.

\bibitem[Jelassi et~al.(2024)Jelassi, Brandfonbrener, Kakade, and Malach]{jelassi2024repeat}
S.~Jelassi, D.~Brandfonbrener, S.~M. Kakade, and E.~Malach.
\newblock Repeat after me: Transformers are better than state space models at copying.
\newblock \emph{arXiv preprint arXiv:2402.01032}, 2024.

\bibitem[Ju et~al.(2021)Ju, Roller, Sukhbaatar, and Weston]{ju2021staircase}
D.~Ju, S.~Roller, S.~Sukhbaatar, and J.~Weston.
\newblock Staircase attention for recurrent processing of sequences.
\newblock \emph{arXiv preprint arXiv:2106.04279}, 2021.

\bibitem[Katharopoulos et~al.(2020)Katharopoulos, Vyas, Pappas, and Fleuret]{katharopoulos2020transformers}
A.~Katharopoulos, A.~Vyas, N.~Pappas, and F.~Fleuret.
\newblock Transformers are rnns: Fast autoregressive transformers with linear attention.
\newblock In \emph{International conference on machine learning}, pages 5156--5165. PMLR, 2020.

\bibitem[Liu et~al.(2023)Liu, Ash, Goel, Krishnamurthy, and Zhang]{liu2022shortcuts}
B.~Liu, J.~Ash, S.~Goel, A.~Krishnamurthy, and C.~Zhang.
\newblock Transformers learn shortcuts to automata.
\newblock In \emph{ICLR}, 2023.
\newblock arXiv:2210.10749.

\bibitem[McCandlish et~al.(2018)McCandlish, Kaplan, Amodei, and Team]{McCandlish2018Empirical}
S.~McCandlish, J.~Kaplan, D.~Amodei, and O.~Team.
\newblock An empirical model of large-batch training.
\newblock \emph{arXiv preprint arXiv:1812.06162}, 2018.

\bibitem[Merrill et~al.(2022)Merrill, Sabharwal, and Smith]{merrill2022saturated}
W.~Merrill, A.~Sabharwal, and N.~A. Smith.
\newblock Saturated transformers are constant-depth threshold circuits.
\newblock \emph{Transactions of the Association for Computational Linguistics}, 10:\penalty0 843--856, 2022.
\newblock \doi{10.1162/tacl_a_00493}.
\newblock URL \url{https://aclanthology.org/2022.tacl-1.49/}.

\bibitem[Mihaylov et~al.(2018)Mihaylov, Clark, Khot, and Sabharwal]{Mihaylov2018}
T.~Mihaylov, P.~Clark, T.~Khot, and A.~Sabharwal.
\newblock Can a suit of armor conduct electricity? a new dataset for open book question answering.
\newblock In \emph{Proceedings of the EMNLP}, 2018.

\bibitem[OLMo et~al.(2024)OLMo, Walsh, Soldaini, Groeneveld, Lo, Arora, Bhagia, Gu, Huang, Jordan, et~al.]{olmo20242}
T.~OLMo, P.~Walsh, L.~Soldaini, D.~Groeneveld, K.~Lo, S.~Arora, A.~Bhagia, Y.~Gu, S.~Huang, M.~Jordan, et~al.
\newblock 2 olmo 2 furious.
\newblock \emph{arXiv preprint arXiv:2501.00656}, 2024.

\bibitem[Oncescu et~al.(2025)Oncescu, Purandare, Idreos, and Kakade]{flashInference}
C.-A. Oncescu, S.~J. Purandare, S.~Idreos, and S.~Kakade.
\newblock Flash inference: Near linear time inference for long convolution sequence models and beyond.
\newblock In Y.~Yue, A.~Garg, N.~Peng, F.~Sha, and R.~Yu, editors, \emph{International Conference on Learning Representations}, volume 2025, pages 49732--49757, 2025.
\newblock URL \url{https://proceedings.iclr.cc/paper_files/paper/2025/file/7c818dd40651b420873af70b8a790e3f-Paper-Conference.pdf}.

\bibitem[Orvieto et~al.(2023)Orvieto, Smith, Gu, Fernando, Gulcehre, Pascanu, and De]{orvieto2023resurrecting}
A.~Orvieto, S.~L. Smith, A.~Gu, A.~Fernando, C.~Gulcehre, R.~Pascanu, and S.~De.
\newblock Resurrecting recurrent neural networks for long sequences.
\newblock In \emph{International Conference on Machine Learning}, pages 26670--26698. PMLR, 2023.

\bibitem[Pascanu et~al.(2013)Pascanu, Mikolov, and Bengio]{pascanu2013difficulty}
R.~Pascanu, T.~Mikolov, and Y.~Bengio.
\newblock On the difficulty of training recurrent neural networks, 2013.

\bibitem[Paszke et~al.(2017)Paszke, Gross, Chintala, Chanan, Yang, DeVito, Lin, Desmaison, Antiga, and Lerer]{paszke2017automaticPyTorch}
A.~Paszke, S.~Gross, S.~Chintala, G.~Chanan, E.~Yang, Z.~DeVito, Z.~Lin, A.~Desmaison, L.~Antiga, and A.~Lerer.
\newblock Automatic differentiation in pytorch.
\newblock In \emph{NIPS-W}, 2017.

\bibitem[Peng et~al.(2023)Peng, Alcaide, Anthony, Albalak, Arcadinho, Biderman, Cao, Cheng, Chung, Grella, et~al.]{peng2023rwkv}
B.~Peng, E.~Alcaide, Q.~Anthony, A.~Albalak, S.~Arcadinho, S.~Biderman, H.~Cao, X.~Cheng, M.~Chung, M.~Grella, et~al.
\newblock Rwkv: Reinventing rnns for the transformer era.
\newblock \emph{arXiv preprint arXiv:2305.13048}, 2023.

\bibitem[Poli et~al.(2024)Poli, Thomas, Nguyen, Ponnusamy, Deiseroth, Kersting, Suzuki, Hie, Ermon, Re, Zhang, and Massaroli]{poli2024mad}
M.~Poli, A.~W. Thomas, E.~Nguyen, P.~Ponnusamy, B.~Deiseroth, K.~Kersting, T.~Suzuki, B.~Hie, S.~Ermon, C.~Re, C.~Zhang, and S.~Massaroli.
\newblock Mechanistic design and scaling of hybrid architectures.
\newblock In \emph{Forty-first International Conference on Machine Learning}, 2024.
\newblock URL \url{https://openreview.net/forum?id=GDp7Gyd9nf}.

\bibitem[Press et~al.(2022)Press, Smith, and Lewis]{press2022trainshorttestlong}
O.~Press, N.~A. Smith, and M.~Lewis.
\newblock Train short, test long: Attention with linear biases enables input length extrapolation, 2022.
\newblock URL \url{https://arxiv.org/abs/2108.12409}.

\bibitem[Rabe and Staats(2021)]{rabe2021blockAttention}
M.~N. Rabe and C.~Staats.
\newblock Self-attention does not need {$ O (n^2) $} memory.
\newblock \emph{arXiv preprint arXiv:2112.05682}, 2021.

\bibitem[Rae et~al.(1911)Rae, Potapenko, Jayakumar, Hillier, and Lillicrap]{rae1911compressive}
J.~W. Rae, A.~Potapenko, S.~M. Jayakumar, C.~Hillier, and T.~P. Lillicrap.
\newblock Compressive transformers for long-range sequence modelling. arxiv preprint, 2019.
\newblock \emph{URL https://arxiv. org/abs}, 1911.

\bibitem[Raffel et~al.(2020)Raffel, Shazeer, Roberts, Lee, Narang, Matena, Zhou, Li, and Liu]{C4raffel2020exploring}
C.~Raffel, N.~Shazeer, A.~Roberts, K.~Lee, S.~Narang, M.~Matena, Y.~Zhou, W.~Li, and P.~J. Liu.
\newblock Exploring the limits of transfer learning with a unified text-to-text transformer.
\newblock \emph{Journal of machine learning research}, 21\penalty0 (140):\penalty0 1--67, 2020.

\bibitem[Sakaguchi et~al.(2020)Sakaguchi, Bras, Bhagavatula, and Choi]{Sakaguchi2020}
K.~Sakaguchi, R.~L. Bras, C.~Bhagavatula, and Y.~Choi.
\newblock Winogrande: An adversarial winograd schema challenge at scale.
\newblock In \emph{Proceedings of the AAAI Conference on Artificial Intelligence}, 2020.

\bibitem[Shallue et~al.(2018)Shallue, Lee, Antognini, Sohl-Dickstein, Frostig, and Dahl]{Shallue2018Measuring}
C.~J. Shallue, J.~Lee, J.~Antognini, J.~Sohl-Dickstein, R.~Frostig, and G.~E. Dahl.
\newblock Measuring the effects of data parallelism on neural network training.
\newblock \emph{arXiv preprint arXiv:1811.03600}, 2018.

\bibitem[Smith et~al.(2022)Smith, Warrington, and Linderman]{smith2022simplified}
J.~T. Smith, A.~Warrington, and S.~W. Linderman.
\newblock Simplified state space layers for sequence modeling.
\newblock \emph{arXiv preprint arXiv:2208.04933}, 2022.

\bibitem[Sun et~al.(2023)Sun, Dong, Huang, Ma, Xia, Xue, Wang, and Wei]{sun2023retentive}
Y.~Sun, L.~Dong, S.~Huang, S.~Ma, Y.~Xia, J.~Xue, J.~Wang, and F.~Wei.
\newblock Retentive network: A successor to transformer for large language models.
\newblock \emph{arXiv preprint arXiv:2307.08621}, 2023.

\bibitem[Vaswani et~al.(2017)Vaswani, Shazeer, Parmar, et~al.]{vaswani2017attention}
A.~Vaswani, N.~Shazeer, N.~Parmar, et~al.
\newblock Attention is all you need.
\newblock In \emph{NeurIPS}, 2017.

\bibitem[Welbl et~al.(2017)Welbl, Liu, and Gardner]{Welbl2017}
J.~Welbl, N.~F. Liu, and M.~Gardner.
\newblock Crowdsourcing multiple choice science questions.
\newblock In \emph{Proceedings of the Workshop on Noisy User-generated Text (WNUT)}, 2017.

\bibitem[Williams et~al.(2009)Williams, Waterman, and Patterson]{williams2009roofline}
S.~Williams, A.~Waterman, and D.~Patterson.
\newblock Roofline: An insightful visual performance model for multicore architectures.
\newblock \emph{Communications of the ACM}, 52\penalty0 (4):\penalty0 65--76, 2009.
\newblock \doi{10.1145/1498765.1498785}.

\bibitem[Xiong et~al.(2020)Xiong, Yang, He, Zheng, Zheng, Xing, Zhang, Lan, Wang, and Liu]{xiong2020layerPreLN}
R.~Xiong, Y.~Yang, D.~He, K.~Zheng, S.~Zheng, C.~Xing, H.~Zhang, Y.~Lan, L.~Wang, and T.~Liu.
\newblock On layer normalization in the transformer architecture.
\newblock In \emph{International conference on machine learning}, pages 10524--10533. PMLR, 2020.

\bibitem[Yang et~al.(2023)Yang, Yu, Zhu, and Hayou]{yang2023tensor}
G.~Yang, D.~Yu, C.~Zhu, and S.~Hayou.
\newblock Tensor programs vi: Feature learning in infinite-depth neural networks, 2023.

\bibitem[Zellers et~al.(2019)Zellers, Bisk, Farhadi, and Choi]{Zellers2019}
R.~Zellers, Y.~Bisk, A.~Farhadi, and Y.~Choi.
\newblock Hellaswag: Can a machine really finish your sentence?
\newblock In \emph{Proceedings of the Annual Meeting of the Association for Computational Linguistics (ACL)}, 2019.

\bibitem[Zhang and Sennrich(2019)]{RMSNorm}
B.~Zhang and R.~Sennrich.
\newblock Root mean square layer normalization.
\newblock In H.~Wallach, H.~Larochelle, A.~Beygelzimer, F.~d\textquotesingle Alch\'{e}-Buc, E.~Fox, and R.~Garnett, editors, \emph{Advances in Neural Information Processing Systems}, volume~32. Curran Associates, Inc., 2019.
\newblock URL \url{https://proceedings.neurips.cc/paper_files/paper/2019/file/1e8a19426224ca89e83cef47f1e7f53b-Paper.pdf}.

\bibitem[Zhang et~al.(2025)Zhang, Morwani, Vyas, Wu, Zou, Ghai, Foster, and Kakade]{zhang2025how}
H.~Zhang, D.~Morwani, N.~Vyas, J.~Wu, D.~Zou, U.~Ghai, D.~Foster, and S.~M. Kakade.
\newblock How does critical batch size scale in pre-training?
\newblock In \emph{The Thirteenth International Conference on Learning Representations}, 2025.
\newblock URL \url{https://openreview.net/forum?id=JCiF03qnmi}.

\bibitem[Zhao et~al.(2025)Zhao, Morwani, Brandfonbrener, Vyas, and Kakade]{zhao2025deconstructing}
R.~Zhao, D.~Morwani, D.~Brandfonbrener, N.~Vyas, and S.~M. Kakade.
\newblock Deconstructing what makes a good optimizer for autoregressive language models.
\newblock In \emph{The Thirteenth International Conference on Learning Representations}, 2025.
\newblock URL \url{https://openreview.net/forum?id=zfeso8ceqr}.

\end{thebibliography}
\bibliographystyle{abbrvnat}
\newpage
\appendix
\onecolumn

\section{Simulating Transformers with \methodname}

\subsection{Transformer Generalization Theorem Statement}
\label{app:transformer-sim}
The approximate part of the Informal Theorem~\ref{thm:infrml-gen} refers to the statement applying exactly when no $\RMS$s are used. This is only a small technicality required to make the statement exact. We restate both architectures without $\RMS$s (including inside attention projections and the MLP) and give an exact representation construction in this setting.

\paragraph{Norm-free architectures.}
\emph{Transformer (width $d'$).} Given inputs $\vx_1^T,\ldots,\vx_N^T\in\R^{d'}$ and parameters
$Q^T,K^T,V^T\in\R^{d'\times d'}$ and $\MLP^T:\R^{d'}\to\R^{d'}$, then the outputs $\vy_1^T,\ldots,\vy_N^T\in\R^{d'}$ are computed by:
\begin{align*}
& \vq_i^T = Q^T\vx_i^T \qquad \vk_i^T = K^T\vx_i^T \qquad \vv_i^T = V^T\vx_i^T \\
& \va_i^T = \Attn\big((\vk_1^T,\vv_1^T),\ldots,(\vk_i^T,\vv_i^T),\vq_i^T\big) \\
& \vy_i^T = \vx_i^T + \va_i^T + \MLP^T[\vx_i^T+\va_i^T] .
\end{align*}

\emph{\methodname{} (width $d$).} Given inputs $\vx_1,\ldots,\vx_N\in\R^{d}$ and parameters
$Q,K,V\in\R^{d\times d}$ and $\MLP:\R^{d}\to\R^{d}$, the $\vy_1,\ldots,\vy_N\in\R^{d}$ are defined via:
\begin{align*}
& \vq_i = Q\vx_i \qquad \vk_i^{\temp} = K\vx_i \qquad \vv_i^{\temp} = V\vx_i \\
& \va_i = \Attn\big((\vk_1,\vv_1),\ldots,(\vk_{i-1},\vv_{i-1}),(\vk_i^{\temp},\vv_i^{\temp}),\vq_i\big) \\
& \vz_i = \vx_i + \va_i + \MLP[\vx_i+\va_i] \\
& \vk_i = K\vz_i \qquad \vv_i = V\vz_i .
\end{align*}

\begin{theorem}[Transformer Generalization]
\label{thm:transformer-containment-normfree}
Assuming neither architecture uses $\RMS$s, any width-$d'$ Transformer (of arbitrary depth) can be simulated by a width-$d=3d'$ \methodname{} of as many layers.
There exists a parameterization of \methodnameabv{} such that Transformer's activations are embedded into disjoint feature groups of the \methodnameabv's ones across layers for any input sequence. This is achieved while ensuring that:
(i) attention scores match those of the Transformer at every position and every layer
and (ii) the layer output exactly tracks the Transformer layer output.
\end{theorem}

\subsection{Proof}
\label{app:transformer-sim-proof}

\paragraph{Three blocks and the per-layer invariant.}
Let $d=3d'$ and decompose $\R^{d}$ into three $d'$-dimensional blocks
\begin{align*}
& \R^{3d'} = \mathcal{C}\oplus\mathcal{L}\oplus\mathcal{S} .
\end{align*}
We call them \emph{carry} ($\mathcal{C}$), \emph{live} ($\mathcal{L}$) and \emph{scratch} ($\mathcal{S}$).
Carry is the only block that $K$ and $V$ read from (so both temporary and persistent key/value pairs depend on it); live holds the next-layer activation and scratch holds attention outputs so residual addition does not corrupt carry.

Fix one Transformer layer (parameters $Q^T,K^T,V^T,\MLP^T$) and assume, for every position $1 \leq i \leq N$ that:
\begin{align*}
& \vx_i = (\vx_i^T \; ; \; * \; ; \; 0) .
\end{align*}
We will construct \methodnameabv's layer parameters ($Q,K,V,\MLP$) so that the output satisfies:
\begin{align*}
& \vz_i = (\vx_i^T \; ; \; \vy_i^T \; ; \; 0) .
\end{align*}
Then a swap between the dimensional blocks $\mathcal{C}$ and $\mathcal{L}$  restores the same input form (namely, the output follows the $(\vy_i^T \; ; \; * \; ; \; 0)$ pattern following the swap) enabling stacking. Define block-sparse linear maps
\begin{align*}
& Q(\vc \; ; \; \vl \; ; \; \vs) = (Q^T\vc \; ; \; 0 \; ; \; 0) \\
& K(\vc \; ; \; \vl \; ; \; \vs) = (K^T\vc \; ; \; 0 \; ; \; 0) \\
& V(\vc \; ; \; \vl \; ; \; \vs) = (0 \; ; \; 0 \; ; \; V^T\vc) .
\end{align*}
Because $K$ and $V$ read only from carry the fact that they are shared between temporary and persistent pairs is automatically respected.

\paragraph{Attention matches (induction on position).}
We prove by induction on $i$ that attention scores match and that the attention output lands in scratch.

Induction hypothesis: for all $j<i$ the persistent pairs match the embedded Transformer pairs
\begin{align*}
& \vk_j = (\vk_j^T \; ; \; 0 \; ; \; 0) \qquad \vv_j = (0 \; ; \; 0 \; ; \vv_j^T) .
\end{align*}
Using the input form $\vx_i = (\vx_i^T \; ; \; * \; ; \; 0)$ we have
\begin{align*}
& \vq_i = (\vq_i^T \; ; \; 0 \; ; \; 0)
\qquad
\vk_i^{\temp} = (\vk_i^T \; ; \; 0 \; ; \; 0)
\qquad
\vv_i^{\temp} = (0 \; ; \; 0 \; ; \vv_i^T) .
\end{align*}
Therefore the logits $\innerp{\vk_j}{\vq_i}$ match $\innerp{\vk_j^T}{\vq_i^T}$ for all $j\le i$ and the attention output matches as well
\begin{align*}
& \va_i = (0 \; ; \; 0 \; ; \va_i^T) .
\end{align*}

It thus follows that $\vx_i+\va_i = (\vx_i^T \; ; \; * \; ; \; \va_i^T)$.
Define $\MLP[\cdot]$ on such inputs so that:
\begin{align*}
& \MLP[(\vx_i^T \; ; \; * \; ; \; \va_i^T)]
=
(0 \; ; \; \vx_i^T + \va_i^T + \MLP^T[\vx_i^T+\va_i^T] - *\; ; \; -\va_i^T) .
\end{align*}
Substituting into the \methodname{} update shows why this choice is natural
\begin{align*}
& \vz_i = \vx_i + \va_i + \MLP[\vx_i+\va_i] \\
& \vz_i
= (\vx_i^T \; ; \; * \; ; \; 0) + (0 \; ; \; 0 \; ; \va_i^T)
+ (0 \; ; \; \vx_i^T + \va_i^T + \MLP^T[\vx_i^T+\va_i^T] - * \; ; \; -\va_i^T) \\
& \vz_i
= (\vx_i^T \; ; \; \vx_i^T + \va_i^T + \MLP^T[\vx_i^T+\va_i^T] \; ; \; 0) \\
& \vz_i = (\vx_i^T \; ; \; \vy_i^T \; ; \; 0) .
\end{align*}
In particular the persistent pairs for position $i$ satisfy
\begin{align*}
& \vk_i = K\vz_i = (\vk_i^T \; ; \; 0 \; ; \; 0)
\qquad
\vv_i = V\vz_i = (0 \; ; \; 0 \; ; \vv_i^T) .
\end{align*}
This closes the induction and proves equality of attention scores at every position within the layer.

\paragraph{Stacking layers.}
After one layer we have $\vz_i = (\vx_i^T \; ; \; \vy_i^T \; ; \; 0)$.
To simulate the next Transformer layer, the next \methodname{} layer must see carry equal to $\vy_i^T$ while scratch remains $0$.
Swapping carry and live between layers, we get:
\begin{align*}
& (\vx_i^T \; ; \; \vy_i^T \; ; \; 0)\mapsto(\vy_i^T \; ; \; \vx_i^T \; ; \; 0) .
\end{align*}
This restores the input form $\vx_i = (\vx_i^T \; ; \; * \; ; \; 0)$ for next layer since its input is current layer's output $\vy_i^T$.
Since each \methodname{} layer has its own parameters the swap can be absorbed into the next layer's $(Q,K,V,\MLP)$ choice.
Iterating over depth completes the simulation of an arbitrary-depth Transformer.

\paragraph{Remark (single layer vs stacking and why $3d'$ is needed).}
A single layer can be simulated with $2d'$ by preserving carry and writing the output elsewhere.
Stacking forces an additional scratch subspace: attention outputs must be representable and cancelable without corrupting the carry block that $K,V$ read while the live block stores the next-layer activation.
This is why the clean exact construction uses $3d'$.

\newpage
\section{Training Stability of \methodname} \label{app:train_stable}
The proof for Theorem \ref{thm:train_stable} is provided below.

\begin{proof}
Given the update, we can see
\[ \frac{\partial z_k}{\partial x_1} = \frac{\alpha}{k} V \sum_{j=1}^{k-1} \frac{\partial z_j}{\partial x_1} \]

Moreover, 
\[ \frac{\partial z_1}{\partial x_1} = I + \alpha V \]
Let's denote $\frac{\partial z_k}{\partial x_1}$ by $f(k)$. 
Define
\[
S(k):=\sum_{j=1}^k f(j).
\]
Then for $k\ge 2$, the recurrence gives
\[
f(k)=\frac{\alpha}{k}V\,S(k-1).
\]
Hence
\begin{align*}
S(k)
&=S(k-1)+f(k)\\
&=S(k-1)+\frac{\alpha}{k}V\,S(k-1)\\
&=\left(I+\frac{\alpha}{k}V\right)S(k-1).
\end{align*}
Since
\[
S(1)=f(1)=I+\alpha V,
\]
we obtain by iterating the above relation that
\[
S(k)=\prod_{m=1}^k \left(I+\frac{\alpha}{m}V\right).
\]
Therefore, for $k\ge 2$,
\begin{align*}
f(k)
&=S(k)-S(k-1)\\
&=\left(I+\frac{\alpha}{k}V\right)S(k-1)-S(k-1)\\
&=\frac{\alpha}{k}V\,S(k-1)\\
&=\frac{\alpha}{k}V\prod_{m=1}^{k-1}\left(I+\frac{\alpha}{m}V\right).
\end{align*}

Now set $z=\alpha V$. Then
\begin{align*}
f(k)
&=\frac{z}{k}\prod_{m=1}^{k-1}\left(I+\frac{z}{m}\right)\\
&=\frac{z}{k}\prod_{m=1}^{k-1}\frac{z+mI}{m}\\
&=\frac{z}{k}\cdot \frac{1}{(k-1)!}\prod_{m=1}^{k-1}(z+mI)\\
&=\frac{1}{k!}\prod_{m=0}^{k-1}(z+mI).
\end{align*}
Substituting back $z=\alpha V$ gives
\[
f(k)=\frac{1}{k!}\prod_{m=0}^{k-1}(\alpha V+mI).
\]

Finally, we use the standard rising-factorial expansion
\[
x(x+1)\cdots(x+k-1)=\sum_{r=0}^k {k \brack r}x^r,
\]
where ${k \brack r}$ are the unsigned Stirling numbers of the first kind or the total number of permutations of $k$ elements with exactly $r$ cycles. Replacing the scalar variable $x$ by the matrix $\alpha V$, we obtain
\[
\prod_{m=0}^{k-1}(\alpha V+mI)=\sum_{r=0}^k {k \brack r}\,\alpha^r V^r.
\]
Since ${k \brack 0}=0$ for $k\ge 1$, this becomes
\[
\prod_{m=0}^{k-1}(\alpha V+mI)=\sum_{r=1}^k {k \brack r}\,\alpha^r V^r.
\]
Therefore,
\[
f(k)=\frac{1}{k!}\sum_{r=1}^k {k \brack r}\,\alpha^r V^r,
\]
as claimed.
\end{proof}

\newpage
\section{More on computational efficiency}
\label{app:complete-algo}
We start by providing the full forward-pass algorithm for evaluating one \methodnameabv{} layer exactly (Algorithm~\ref{alg:rt-fwd}).
The schedule follows Figure~\ref{fig:tiling}: persistent key--value pairs $(\vk_t,\vv_t)$ are revealed sequentially (only after $\vz_t$ is computed), but queries $\{\vq_i\}_{i=1}^N$ are available from the very beginning.
We take advantage of this by immediately attenting to newly-available key--value pairs across an entire range of future queries rather than only the next query, increasing KV-access reuse while preserving the model's exact computation.

Concretely, when token $t$ finishes, the tile size is chosen as the largest power of $2$, $P$ that divides $t$.
Algorithm~\ref{alg:rt-fwd} immediately incorporates the contribution of $(\vk_{t-P+1:t},\vv_{t-P+1:t})$ into the attention accumulators of the next query block $q_{t+1:t+P}$.
Over the full run, every query position accumulates contributions from every earlier key--value pair exactly once, matching naive causal attention up to floating-point reordering effects.

\begin{algorithm}[t]
\caption{Exact tiled forward pass for one \methodname{} layer (training/prefill)}
\label{alg:rt-fwd}
\begin{algorithmic}[1]
\REQUIRE Inputs $\vx_{1:N}$ for a single layer
\REQUIRE Projections $Q,K,V$ and MLP block $\MLP$
\REQUIRE Query tile size $B$ (power of $2$)
\ENSURE Outputs $\vz_{1:N}$ and persistent pairs $(\vk_{1:N},\vv_{1:N})$

\STATE Compute queries in parallel: $\vq_i \leftarrow \qkRMS[Q\,\RMS(\vx_i)]$ for $i=1,\ldots,N$
\STATE Initialize running stats for all queries:
\STATE $m_i \leftarrow -\infty$ \quad $l_i \leftarrow 0$ \quad $\vo_i \leftarrow 0$ for $i=1,\ldots,N$
\STATE Initialize persistent buffers $\vk_{1:N},\vv_{1:N}$ as empty

\FOR{$t=1: N$}
    \STATE $\vk_t^{\temp} \leftarrow \qkRMS[K\,\RMS(\vx_t)]$ \qquad $\vv_t^{\temp} \leftarrow V\,\RMS(\vx_t)$
    \STATE $\textsc{UpdateTile}\big(q_{t:t}\;,\; \vk_t^{\temp}\;,\;\vv_t^{\temp}\big)$ \COMMENT{temporary self contribution}

    \STATE $\va_t \leftarrow \vo_t / l_t$
    \STATE $\vz_t \leftarrow \vx_t + \va_t + \MLP[\RMS(\vx_t+\va_t)]$
    \STATE $\vk_t \leftarrow \qkRMS[K\,\RMS(\zmark{\vz_t})]$ \qquad $\vv_t \leftarrow V\,\RMS(\zmark{\vz_t})$ \COMMENT{persistent KV pair revealed}

    \STATE $P \leftarrow 2^{\nu_2(t)}$ \COMMENT{largest power of $2$ dividing $t$}
    \STATE $(u, v) \leftarrow (t+1, \min(t+P,N)))$
    \IF{$u \le v$}
        \STATE $\textsc{UpdateTile}\big(q_{u:v}\;,\;\vk_{t-P+1:t}\;,\;\vv_{t-P+1:t}\big)$
        \COMMENT{have the next query block attend to the newly-available KV segment}
    \ENDIF
\ENDFOR
\end{algorithmic}
\end{algorithm}

\begin{algorithm}
\caption{$\textsc{UpdateTile}(q_{u:v},\vk_{s:e},\vv_{s:e})$: online-softmax update for a query tile}
\label{alg:update-tile}
\begin{algorithmic}[1]
\REQUIRE Query indices $u\!:\!v$ with queries $\vq_{u:v}$
\REQUIRE A key--value tile $\vk_{s:e},\vv_{s:e}$ (persistent) or a single temporary pair $(\vk_t^{\temp},\vv_t^{\temp})$
\REQUIRE Running stats $(m_{u:v},l_{u:v},\vo_{u:v})$
\ENSURE Updated $(m_{u:v},l_{u:v},\vo_{u:v})$ corresponding to including this tile

\STATE Compute tile logits $\alpha_{i,j} \leftarrow \innerp{\vq_i}{\vk_j}$ for all $i\in[u,v]$ and $j\in[s,e]$
\STATE Compute per-query tile maxima $m^{\text{tile}}_i \leftarrow \max_{j\in[s,e]} \alpha_{i,j}$ for all $i\in[u,v]$
\STATE Compute new maxima $m^{\text{new}}_i \leftarrow \max(m_i,\; m^{\text{tile}}_i)$ for all $i\in[u,v]$

\STATE Rescale old accumulators:
\STATE $\vo_i \leftarrow \vo_i \cdot \exp(m_i - m^{\text{new}}_i)$ \qquad $l_i \leftarrow l_i \cdot \exp(m_i - m^{\text{new}}_i)$ for all $i\in[u,v]$

\STATE Accumulate this tile:
\STATE $\vo_i \leftarrow \vo_i + \sum_{j=s}^e \vv_j \exp(\alpha_{i,j} - m^{\text{new}}_i)$ for all $i\in[u,v]$
\STATE $l_i \leftarrow l_i + \sum_{j=s}^e \exp(\alpha_{i,j} - m^{\text{new}}_i)$ for all $i\in[u,v]$

\STATE Finalize maxima: $m_i \leftarrow m^{\text{new}}_i$ for all $i\in[u,v]$
\end{algorithmic}
\end{algorithm}

Regarding how to accumulate attention contributions from multiple ranges of key--value pairs, we use \citep{rabe2021blockAttention,dao2022flashattention}'s approach; For each query position (or query tile), we maintain the standard online-softmax running statistics:
\begin{itemize}
    \item a running max logit $m$
    \item a running normalizer $l$
    \item and a running numerator vector $\vo$
\end{itemize}
When a new contribution tile is processed, these statistics are updated by rescaling the existing accumulators and adding the tile's contribution computed relative to the updated maximum (keeping the logit maxima is only required for numerical stability). After all prefix tiles have been incorporated for position $t$, the exact attention output is recovered as $\va_t=\vo_t/l_t$.

Algorithm~\ref{alg:update-tile} spells out the $\textsc{UpdateTile}$ primitive used by the forward schedule. It takes a query range and a range of key--value pairs and updates $(m, l,\vo)$ for all queries in the tile in a vectorized manner.

\newpage
\section{Hyperparameter details} \label{app:hyper}

\subsection{C4 pretraining experiments} \label{app:hyper_c4}
For the C4 pretraining experiments in Figure \ref{fig:c4-300m}, we used a 300m non-embedding parameter transformer. For the 12 layer experiments, the width of the model was $1408$, with the MLP width being $5632$ and number of heads being $22$ (so as to keep per-head dim to be $64$). For the 6 layer experiments, the width was adjusted to $2048$, with the MLP width being $8192$ and number of heads being $32$. The maximum sequence length was fixed to $512$ and the models were trained for $1$x Chinchilla tokens ($\approx 6b$ tokens), leading to $25 k$ steps for the $512$ batch size experiment. We used the alibi positional embeddings \citep{press2022trainshorttestlong} with max alibi bias of $8.0$. The throughput of \methodname{} at 12 layers was $42k$ tokens/sec as compared to $132k$ tokens/sec for vanilla transformer. The throughput of \methodname{} at 6 layers was $49k$ tokens/sec as compared to $153k$ tokens/sec for vanilla transformer. 

We used the Adam optimizer for the experiments, with hyperparameter tuning given by: $\eta \in \{1e-3, 3e-3, 1e-2\}, \beta_1 = 0.9, \beta_2 \in \{0.95, 0.99\}, \eps = 1e-8$, and the weight decay was set to 0.0. We used warmup and cosine schedule for the experiments, with the warmup accounting for $40\%$ of the training as found to be optimal for this scale in previous works \citep{zhao2025deconstructing}.

\subsection{Synthetic experiments} \label{app:hyper_synth}
For the synthetic experiments in Figure \ref{fig:synth}, we used a single layer transformer, with model width $128$, MLP width $512$ and $16$ heads as in \citet{poli2024mad}. We used the alibi positional embeddings with max alibi bias set to $8.0$. We used the AdamW optimizer, with the hyperparameter tuning given by: $\eta \in \{1e-4, 5e-4, 1e-3, 5e-4\}, \beta_1 = 0.9, \beta_2 = 0.98, \epsilon = 1e-8, \lambda \in \{ 0.0, 0.1 \}$, where $\lambda$ represents the weight decay.

\section{More experiments and results}

\subsection{Synthetics Token Level}
\label{app:synthetics-token-level}
Figure~\ref{fig:synth-token-acc} shows the token-level accuracies for the different synthetic tasks. Note how in the compression task where neither transformers nor the \methodnameabv{} have non-trivial performance at sequence level, the accuracy becomes non-trivial at the token level and the gap between the two architectures is still prominent.
\begin{figure*}[t]
  \centering
  \includegraphics[width=\textwidth]{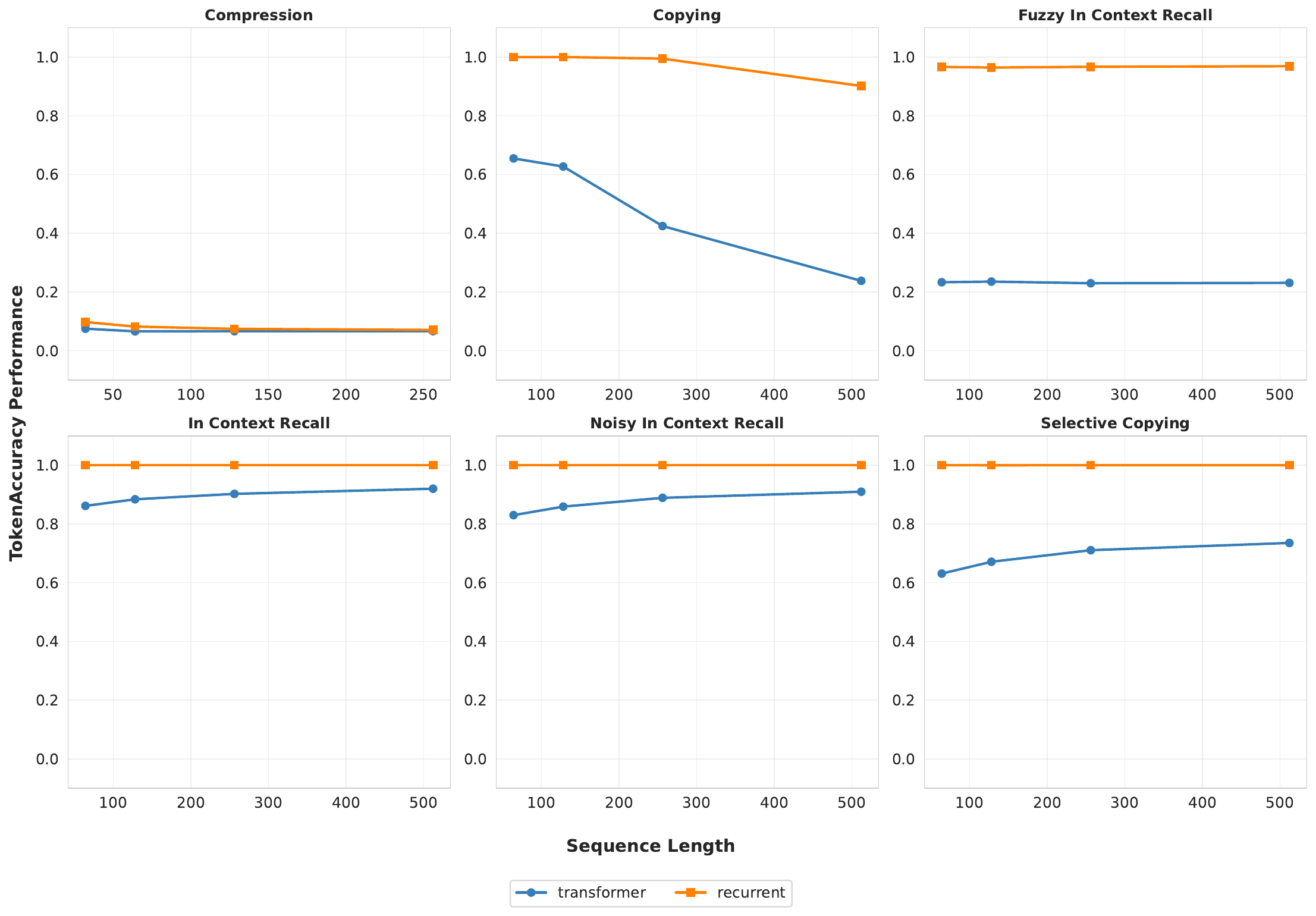}
  \caption{Token level accuracies on synthetic diagnostics (MAD + copy).}
  \label{fig:synth-token-acc}
\end{figure*}

\subsection{RMSNorm Ablation} \label{app:layernorm_stable}
In this section, we ablate the RMSnorm used in \methodname{}, i.e, we replace Equation \ref{eq:persistentK} and \ref{eq:persistentV} with

\begin{align*}
\vk_i &= \qkRMS(K\,\zmark{\vz_i}) \\
\vv_i &= V\,\zmark{\vz_i}
\end{align*}

The best performance with this setup was obtained with $\eta = 1e-3$, with higher learning rates destabilizing. Note that, in comparison, with RMSNorm, even learning rates till $1e-2$ are stable, although, $3e-3$ turns out to be the optimal. The results are shown in Figure \ref{fig:LN-ablate}. As can be seen, the performance obtained is significantly worse without the layernorm.

\begin{figure}[t]
  \centering
  \includegraphics[width=0.7\linewidth]{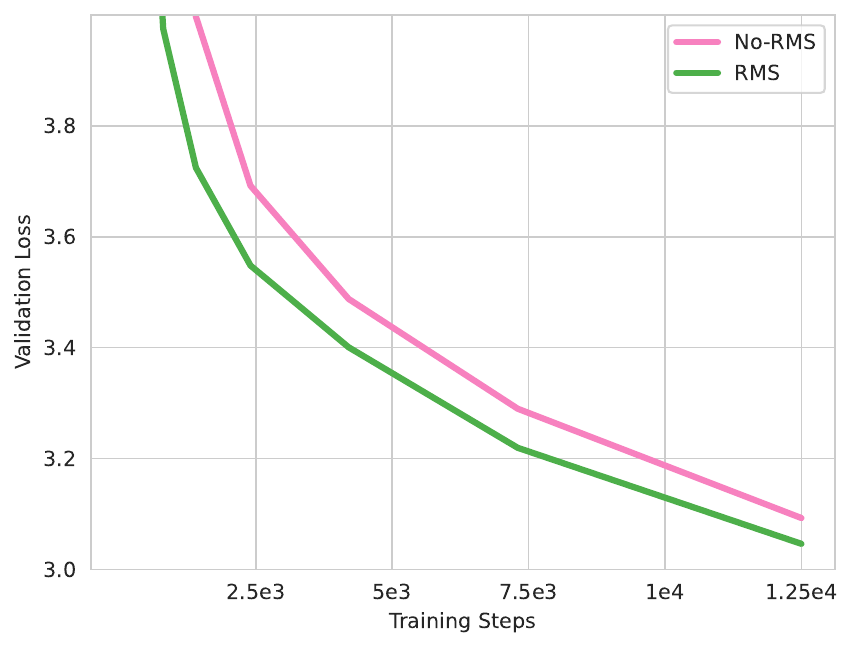}
  \caption{C4 pretraining: Ablating the use of RMSNorm in \methodname{} for 150M parameter model at 512 batch size.}
  \label{fig:LN-ablate}
\end{figure}

\subsection{C4 pretraining (150M scale)} \label{app:150m-pretrain}
For the 150M parameter model, for the 12 layer experiments, the width of the model was 1024, with the MLP width being 4096 and 16 number of heads. For the 6 layer experiments, the width was adjusted to 1408, with the MLP width being 5632 and 22 number of heads. The model was trained for 1x Chinchilla tokens ($\approx$ 3b tokens). We provide loss curves in Figure~\ref{fig:c4-512} and \ref{fig:c4-256} at batch sizes $512$ and $256$ respectively - this is such that the critical batch size of $256K$ of the 150M model is not exceeded \citep{zhang2025how}, while keeping the batch size large enough to have MLPs be compute-bound. Figures~\ref{fig:c4-512} and \ref{fig:c4-256} shows that we have benefits at various batch sizes. The corresponding losses are displayed in Tables~\ref{tab:c4-512} and \ref{tab:c4-256}. We also report the downstream performance of these models in terms of cross entropy loss of the ground truth answer in Tables \ref{tab:c4-512-downstream-ce} and \ref{tab:c4-256-downstream-CE} respectively. We also report the downstream accuracy in Tables \ref{tab:c4-512-downstream-acc} and \ref{tab:c4-256-downstream-acc} respectively.

\begin{table}[t]
\centering
\caption{C4 pretraining loss at 150M parameters at batch size 512.}
\label{tab:c4-512}
\begin{tabular}{lccc}
\toprule
Model & Layers & Width & Val CE $\downarrow$ \\
\midrule
Transformer & 6 & $1408$ & $3.097$ \\
Transformer & 12 & $1024$ & $3.067$ \\
\methodname{} & 6 & $1408$ & $3.049$ \\
\methodname{} & 12 & $1024$ & $3.046$ \\
\bottomrule
\end{tabular}
\end{table}

\begin{figure}[t]
  \centering
  \includegraphics[width=0.7\linewidth]{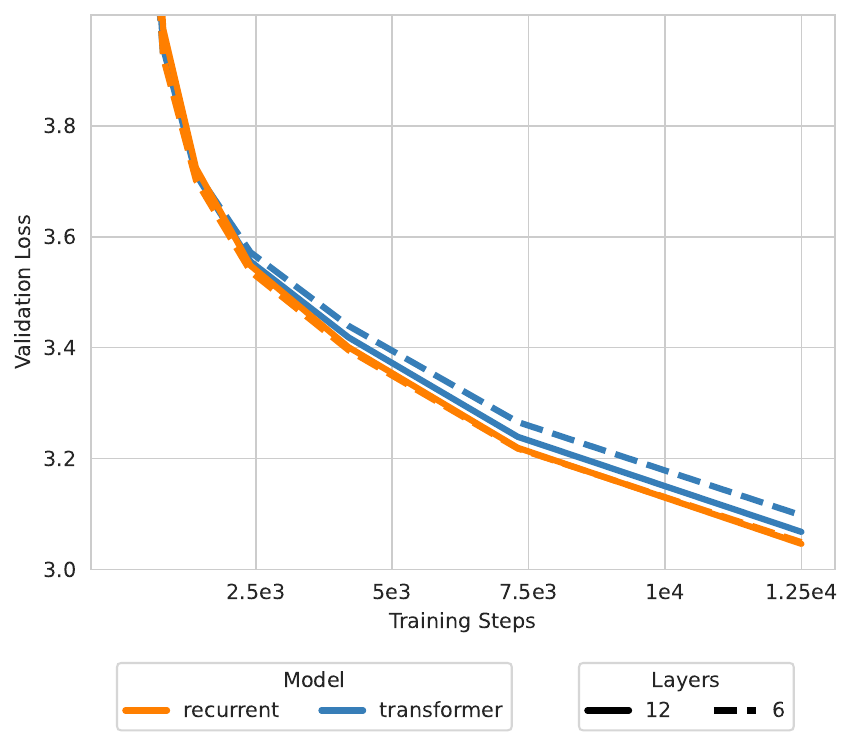}
  \caption{C4 pretraining: loss curve for the 150M parameter model at batch size 512.}
  \label{fig:c4-512}
\end{figure}

\begin{table}[t]
\centering
\caption{Downstream performance for the 150M model at batch size 512.}
\label{tab:c4-512-downstream-ce}
\begin{tabular}{lcccccccc}
\toprule
Model & Layers & piqa CE & hellaswag CE & arc easy CE & openbook qa CE & sciq CE & winogrande CE  \\
\midrule
Transformer & 6 & $5.791$ & $4.460$ & $11.419$ & $12.509$ & $15.3$ & $7.912$ \\
Transformer & 12 & $5.911$ & $4.371$ & $11.87$ & $13.471$ & $14.948$ & $8.365$ \\
Recurrent & 12 & $\textbf{5.606}$ & $\textbf{4.274}$ & $\textbf{11.124}$ & $\textbf{12.067}$ & $\textbf{14.43}$ & $\textbf{7.645}$ \\
Recurrent & 6 & $5.689$ & $4.327$ & $11.242$ & $12.252$ & $14.529$ & $8.314$ \\
\bottomrule
\end{tabular}
\end{table}
\begin{table}[t]
\centering
\caption{Downstream accuracy for the 150M model at batch size 512.}
\label{tab:c4-512-downstream-acc}
\begin{tabular}{lcccccccc}
\toprule
Model & Layers & piqa acc & hellaswag acc & arc easy acc & openbook qa acc & sciq acc & winogrande acc  \\
\midrule
Transformer & 6 & $60.55$ & $28.34$ & $32.11$ & $35.6$ & $34.29$ & $49.49$ \\
Transformer & 12 & $60.94$ & $28.41$ & $31.4$ & $37.4$ & $\textbf{37.17}$ & $\textbf{52.01}$ \\
Recurrent & 12 & $\textbf{61.26}$ & $\textbf{29.67}$ & $\textbf{32.63}$ & $36.2$ & $36.28$ & $50.82$ \\
Recurrent & 6 & $61.21$ & $29.06$ & $31.23$ & $\textbf{38.4}$ & $31.08$ & $48.93$ \\
\bottomrule
\end{tabular}
\end{table}

\begin{table}[t]
\centering
\caption{C4 pretraining loss at 150M parameters, training at batch-size 256.}
\label{tab:c4-256}
\begin{tabular}{lccc}
\toprule
Model & Layers & Width & Val CE $\downarrow$ \\
\midrule
Transformer & 6 & $1440$ & $3.091$ \\
Transformer & 12 & $1024$ & $3.059$ \\
\methodname{} & 6 & $1440$ & $3.037$ \\
\methodname{} & 12 & $1024$ & $3.036$ \\
\bottomrule
\end{tabular}
\end{table}

\begin{figure}[t]
  \centering
  \includegraphics[width=0.7\linewidth]{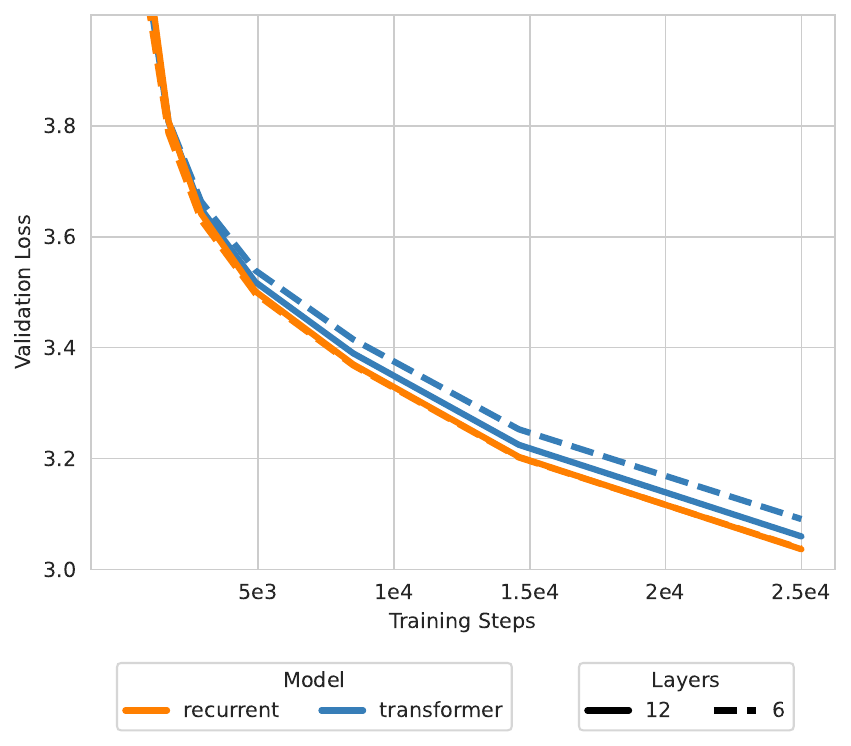}
  \caption{C4 pretraining: loss curve for the 150M parameter model at batch size 256.}
  \label{fig:c4-256}
\end{figure}

\begin{table}[t]
\centering
\caption{Downstream performance for the 150M model at batch size 256.}
\label{tab:c4-256-downstream-CE}
\begin{tabular}{lcccccccc}
\toprule
Model & Layers & piqa CE & hellaswag CE & arc easy CE & openbook qa CE & sciq CE & winogrande CE  \\
\midrule
Transformer & 6 & $5.851$ & $4.43$ & $11.792$ & $13.42$ & $15.061$ & $8.359$ \\
Transformer & 12 & $6.018$ & $4.535$ & $12.552$ & $13.999$ & $16.643$ & $8.883$ \\
Recurrent & 12 & $5.512$ & $4.231$ & $10.822$ & $12.28$ & $14.149$ & $7.842$ \\
Recurrent & 6 & $\textbf{5.387}$ & $\textbf{4.15}$ & $\textbf{10.516}$ & $\textbf{11.529}$ & $\textbf{13.494}$ & $\textbf{7.682}$ \\
\bottomrule
\end{tabular}
\end{table}
\begin{table}[t]
\centering
\caption{Downstream accuracy for the 150M model at batch size 256.}
\label{tab:c4-256-downstream-acc}
\begin{tabular}{lcccccccc}
\toprule
Model & Layers & piqa acc & hellaswag acc & arc easy acc & openbook qa acc & sciq acc & winogrande acc  \\
\midrule
Transformer & 6 & $60.94$ & $28.59$ & $31.58$ & $\textbf{37.8}$ & $39.16$ & $51.3$ \\
Transformer & 12 & $60.5$ & $28.57$ & $32.63$ & $37.6$ & $36.5$ & $\textbf{52.88}$ \\
Recurrent & 12 & $61.92$ & $29.18$ & $31.75$ & $36.6$ & $35.29$ & $49.8$ \\
Recurrent & 6 & $\textbf{62.19}$ & $\textbf{29.28}$ & $\textbf{32.81}$ & $\textbf{37.8}$ & $\textbf{39.71}$ & $50.51$ \\
\bottomrule
\end{tabular}
\end{table}

\subsection{Downstream accuracy of 300M parameter transformer}
\begin{table}[t]
\centering
\caption{Downstream accuracy for the 300M model.}
\label{tab:c4-512-300m-downstream-acc}
\begin{tabular}{lccccccc}
\toprule
Model & Layers & piqa acc & hellaswag acc & arc easy acc & openbook qa acc & sciq acc & winogrande acc  \\
\midrule
Transformer & 6 & $\textbf{63.49}$ & $31.67$ & $\textbf{33.51}$ & $37.8$ & $34.4$ & $50.36$ \\
Transformer & 12 & $62.4$ & $31.82$ & $32.46$ & $\textbf{38.2}$ & $34.4$ & $\textbf{51.85}$ \\
Recurrent & 12 & $\textbf{63.49}$ & $\textbf{33.02}$ & $31.58$ & $37$ & $\textbf{35.95}$ & $50.59$ \\
Recurrent & 6 & $63.27$ & $32.66$ & $30.88$ & $37.2$ & $35.62$ & $49.96$ \\
\bottomrule
\end{tabular}
\end{table}
In Table \ref{tab:c4-512-300m-downstream-acc}, we provide the downstream accuracy for the 300m model.

\end{document}